\documentclass[a4paper, 10pt, onecolumn]{article}  

\usepackage[utf8x]{inputenc}
\usepackage{amsthm}
\usepackage{amssymb}
\usepackage{amsmath}
\usepackage{subfigure}
\usepackage{epsfig}
\usepackage{algorithmic}
\usepackage{algorithm} 
\usepackage{authblk}

\newcommand{\cv}{\textbf{c}}
\newcommand{\dv}{\textbf{d}}

\newcommand{\pv}{\textbf{p}}
\newcommand{\qv}{\textbf{q}}

\newcommand{\uv}{\textbf{u}}
\newcommand{\vv}{\textbf{v}}

\newcommand{\xv}{\textbf{x}}

\newcommand{\Deltam}{\mbox{\boldmath$\Delta$}}
\newcommand{\Gammam}{\mbox{\boldmath$\Gamma$}}
\newcommand{\Omegam}{\mbox{\boldmath$\Omega$}}
\newcommand{\Lambdam}{\mbox{\boldmath$\Lambda$}}
\newcommand{\Sigmam}{\mbox{\boldmath$\Sigma$}}

\newcommand{\Pim}{\mbox{\boldmath$\Pi$}}
\newcommand{\Thetam}{\mbox{\boldmath$\Theta$}}

\newcommand{\diag}{\mbox{diag}}
\newcommand{\tr}{\mbox{tr}}

\newcommand{\Am}{\textbf{A}}
\newcommand{\Bm}{\textbf{B}}
\newcommand{\Cm}{\textbf{C}}
\newcommand{\Dm}{\textbf{D}}
\newcommand{\Em}{\textbf{E}}
\newcommand{\Fm}{\textbf{F}}
\newcommand{\Gm}{\textbf{G}}
\newcommand{\Hm}{\textbf{H}}
\newcommand{\Imat}{\textbf{I}}

\newcommand{\Km}{\textbf{K}}
\newcommand{\Lm}{\textbf{L}}
\newcommand{\Mm}{\textbf{M}}

\newcommand{\Pm}{\textbf{P}}
\newcommand{\Qm}{\textbf{Q}}
\newcommand{\Rm}{\textbf{R}}
\newcommand{\Sm}{\textbf{S}}
\newcommand{\Tm}{\textbf{T}}
\newcommand{\Um}{\textbf{U}}
\newcommand{\Vm}{\textbf{V}}
\newcommand{\Wm}{\textbf{W}}
\newcommand{\Xm}{\textbf{X}}
\newcommand{\Ym}{\textbf{Y}}
\newcommand{\Zm}{\textbf{Z}}

\newtheorem{theorem}{Theorem}[section]
\newtheorem{lemma}{Lemma}[section]
\newtheorem{definition}{Definition}[section]

\begin{document}

\title{Efficient Eigen-updating for Spectral Graph Clustering}

\author[1, *]{Charanpal Dhanjal} 
\author[2]{Romaric Gaudel} 
\author[3]{St\'ephan Cl\'{e}men\c{c}on}

\affil[1]{LIP6, UPMC, 4 Place Jussieu, 75252 Paris Cedex 05, France} 
\affil[2]{Universit\'{e} Lille 3, Domaine Universitaire du Pont de Bois, 59653 Villeneuve d'Ascq Cedex, France}
\affil[3]{Telecom ParisTech, 46 rue Barrault, 75634 Paris Cedex 13, France}

\date{\today}
\maketitle

\let\oldthefootnote\thefootnote
\renewcommand{\thefootnote}{\fnsymbol{footnote}}
\footnotetext[1]{Author for correspondence (charanpal.dhanjal@lip6.fr)}
\let\thefootnote\oldthefootnote

\begin{abstract}
Partitioning a graph into groups of vertices such that those within each group are more densely connected than vertices assigned to different groups, known as \textit{graph clustering}, is often used to gain insight into the organisation of large scale networks and for visualisation purposes. Whereas a large number of dedicated techniques have been recently proposed for static graphs, the design of \textit{on-line} graph clustering methods tailored for evolving networks is a challenging problem, and much less documented in the literature. Motivated by the broad variety of applications concerned, ranging from the study of biological networks to the analysis of networks of scientific references through the exploration of communications networks such as the World Wide Web, it is the main purpose of this paper to introduce a novel, computationally efficient, approach to graph clustering in the evolutionary context. Namely, the method promoted in this article can be viewed as an incremental eigenvalue solution for the spectral clustering method described by Ng. \textit{et al.} (2001). The incremental eigenvalue solution is a general technique for finding the approximate eigenvectors of a symmetric matrix given a change. As well as outlining the approach in detail, we present a theoretical bound on the quality of the approximate eigenvectors using perturbation theory. We then derive a novel spectral clustering algorithm called \textit{Incremental Approximate Spectral Clustering} (IASC). The IASC algorithm is simple to implement and its efficacy is demonstrated on both synthetic and real datasets modelling the evolution of a HIV epidemic, a citation network and the purchase history graph of an e-commerce website. 
\end{abstract}

\section{Introduction}

\textit{Graph-mining} has recently received increasing attention in the machine-learning literature, motivated by application domains such as the Internet, social networks, epidemics of transmissible infectious diseases, sensor and biological networks. A crucial task in exploratory analysis and data visualisation is \textit{graph clustering} \cite{schaeffer2007graph, fortunato2010community}, which aims to partition the vertices in a graph into groups or clusters, with dense internal connections and few connections between each other. There is a large body of work on graph clustering. A possible approach is to consider a certain measure that quantifies community structure and formulate the clustering issue as an optimisation problem (which is generally NP-hard), for which fairly good solutions can be obtained recursively or by using adequate metaheuristics, see \cite{newman2004finding, newman04fastAlg, flake2004graph, satuluri2009scalable} for example. A variety of approaches to graph clustering exist, such as those based on modularity maximisation for instance, see \cite{newman2004finding, newman2006modularity}. In this paper, focus is on the \textit{spectral clustering} approach \cite{von2007tutorial}, which performs well empirically, is often simple to implement and benefits computationally from the availability of fast linear algebra libraries. The general principle of spectral clustering is to compute the smallest eigenvectors of some particular matrix $\Lm$ (refer to Section \ref{sec:spectral_clustering} for further details) and then cluster the vertices based on their representation in the corresponding eigen-space. The popularity of this approach arises from the fact that the obtained clustering of vertices is closely connected to the spectral relaxation of the minimisation of the normalised cut criterion, see \cite{shi2000normalized}.

In many applications such as communications networks (e.g. the Web and Internet), biological networks (of proteins, metabolic reactions, \textit{etc.}), social networks or networks of scientific citations for instance, the graphs of interest slowly change with time. A naive approach to this incremental problem is to cluster each graph in the sequence separately, however this is computationally expensive for spectral clustering as the cost of solving the eigenvalue problem is $\mathcal{O}(n^3)$ at each iteration, where $n$ is the number of vertices. There has been some previous work on the incremental spectral clustering problem, for example \cite{valgren2007incremental, ning2010incremental,ning2007incremental, kong2011fast} however only \cite{ning2010incremental,ning2007incremental} update the eigen-system. In this paper we propose an efficient method for clustering a sequence of graphs which leverages the eigen-decomposition and the clustering on the previous graph to find the new clustering. Firstly, a fast approximation of a rank-$k$ eigen-decomposition of $\Lm_{t+1}$ knowing that of $\Lm_t$ is derived from the Singular Value Decomposition (SVD) updating approach used for Latent Semantic Indexing in \cite{zha1999updating}. Here, the update is efficient to compute when the change (defined in the subsequent analysis) between $\Lm_t$ and $\Lm_{t+1}$ is small. Secondly the clustering of vertices is updated according to the new eigen-space. The efficiency of the complete approach, in terms of clustering accuracy, is demonstrated on synthetic and real data. We point out that the clustering approach was first outlined in \cite{dhanjal11cluster}. Here we provide a theoretical analysis on the quality of the eigen-approximation, as well as a more extensive empirical study of the algorithm.  
\medskip

\par The paper is organised as follows. Standard spectral clustering and SVD updating approaches are recalled in Sections \ref{sec:spectral_clustering} and \ref{sec:svd_update}. Then Section \ref{sec:our_work} details the proposed eigen-decomposition update and Section \ref{sec:quality}   studies the accuracy of the resulting approximate eigenvectors using perturbation theory. In Section \ref{sec:incrementalCluster} we show how the eigen-decomposition updates can be applied to spectral clustering. Numerical results are gathered in Section \ref{sec:exp}, and the paper ends with Section \ref{sec:conclusion} discussing results and planned future work.
\medskip

\noindent\textbf{Notation}: A bold uppercase letter represents a matrix, e.g. $\Xm$, and a column vector is displayed using a bold lowercase letter, e.g. $\xv$. The transpose of a matrix or vector is written $\Xm^T$. The concatenation of the columns of two matrices $\Am$ and $\Bm$ is written $[\Am \quad \Bm]$. $\Am[I, I]$ represent the submatrix of $\Am$ formed using the row and columns indexed by $I$, and $\Am[I, :]$ and $\Am[:, I]$ are submatrices formed using the rows and columns $I$ respectively. The matrix $\Am_k$ is that formed using the largest $k$ eigenvectors and eigenvalues of $\Am$, however $\Imat_p$ is the $p\times p$ identity matrix.  

\section{Graph Clustering}
\label{sec:spectral_clustering}

Consider an undirected graph $G = (V, E)$, composed of a set of vertices $V=\{v_1, \ldots, v_n\}$ and edges $E \subseteq V \times V$ such that for every edge $(v_i, v_j)$ there is also an edge $(v_j, v_i)$. One way of representing the edges of $G$ is using an \emph{adjacency matrix} $\Am \in \{0, 1\}^{n \times n}$ which has $\Am_{ij} = 1$ if there is an edge from the $i$-th to the $j$-th vertex and $\Am_{ij} = 0$ otherwise. More generally, the \emph{weight matrix} $\Wm \in \mathbb{R}^{n \times n}$ allows one to assign nonzero numerical values on edges, and thus $\Wm_{ij} \neq 0$ when there is an edge  from the $i$-th to the $j$-th vertex.

In the perspective of spectral clustering, a useful way of representing $G$ is through its Laplacian matrix \cite{chung1997spectral}. There are several definitions of the Laplacian matrix, however we are interested in the \emph{normalised Laplacian matrix}, which is symmetric and positive semi-definite. We recall its definition below for clarity.

\begin{definition} 
The \emph{unnormalised Laplacian matrix} of a graph $G$ is defined as $\Lm = \Dm - \Wm$ where $\Dm$ is the degree matrix with supposedly nonzero diagonal entries $\Dm_{ii} = deg(v_i)$, denoting by $deg(v_i)=\sum_{j}\Wm_{ij}$ the degree of the $i$-th vertex,  and zeros elsewhere, and $\Wm$ is the weight matrix. The normalised Laplacian matrix of a graph $G$ is then defined as 
$$
\tilde{\Lm} = \Dm^{-\frac{1}{2}}\Lm\Dm^{-\frac{1}{2}}.$$
\end{definition}

The normalised Laplacian matrix is used in the spectral clustering approach of Ng et al. \cite{ng2001spectral}. The algorithm computes the Laplacian and then finds the $k$ smallest eigenvectors which are used for clustering in conjunction with the $k$-means algorithm, see Algorithm \ref{alg:normCluster}. The Laplacian matrix is often sparse and one can use power or Krylov subspace methods such as the Lanczos method to find the eigenvectors. There are several variants of Algorithm \ref{alg:normCluster}, such as that of \cite{shi2000normalized} which uses the so-called random walk Laplacian and clusters the smallest eigenvectors in a similar way. One of the motivations for these clustering methods is from the spectral relaxation of the minimisation of the normalised cut criterion. 

\begin{algorithm}
\caption{Spectral Clustering using the Normalised Laplacian \cite{ng2001spectral}}
\begin{algorithmic}[1]
\REQUIRE Graph $G$ with weight matrix $\Wm \in \mathbb{R}^{n \times n}$, number of clusters $k$ 
\STATE Find $k$ smallest eigenvectors $\Vm_k = [\vv_1, \ldots, \vv_k]$ of normalised Laplacian $\tilde{\Lm}$
\STATE Normalise the rows of $\Vm_k$, i.e. $\Vm_k \leftarrow \diag(\Vm_k\Vm_k^T)^{-\frac{1}{2}}\Vm_k$ 
\STATE Cluster rows of $\Vm_k$ with the $k$-means algorithm 
\RETURN Cluster membership vector $\cv \in \{1,\ldots,k\}^n$  
\end{algorithmic}\label{alg:normCluster}
\end{algorithm}

A naive approach to spectral clustering on a sequence of graphs has a large update cost due to the computation of the eigen-decomposition of $\tilde{\Lm}_t$ at each iteration $t$. A more efficient approach is to update relevant eigenvectors from one iteration to the next. In \cite{valgren2007incremental} the authors use the spectral clustering of \cite{ng2001spectral} however they do not update the eigenvectors incrementally but instead the clustering directly. In \cite{kong2011fast} the authors cluster using the unnormalised Laplacian matrix. Eigenvalues are updated and then clusters are modelled using a set of representative points (a similar strategy is used in \cite{yan2009fast} for example to cluster points in $\mathbb{R}^d$). The algorithm is potentially costly since it uses the full eigen-decomposition in conjunction with the eigen-gap heuristic to estimate the number of clusters. Furthermore, the solution is updated with the addition of vertices only, and not edges. The iterative clustering approach of \emph{Ning et al.} \cite{ning2010incremental,ning2007incremental} uses the spectral clustering method given in \cite{shi2000normalized} and updates eigenvectors incrementally. The algorithm incrementally updates the solution of the generalised eigenproblem $\Lm \vv = \lambda \Dm \vv$ by finding the derivatives on the eigenvalues/vectors with perturbations in all of the quantities involved. An iterative refinement algorithm is given for the eigenvalues and eigenvectors given a change in the edges or vertices of a graph. One then clusters the resulting $k$ smallest eigenvectors using $k$-means clustering. In order to limit errors which can build up cumulatively the authors recompute the eigenvectors after every $R$-th graph in the sequence. 

A disadvantage of the approach of \cite{ning2010incremental,ning2007incremental} lies in the fact that, to update an eigenvector, one must invert a small matrix for each weight change in the graph which makes updates costly. The size of this matrix is proportional to the number of neighbours of the vertices incident to the changed edge. To be more precise, the cost of updating an eigenvector in \cite{ning2010incremental} is $\mathcal{O}(\bar{N}^2 n + \bar{N}^3)$ in which $\bar{N}$ is the average size of the ``spatial neighbourhood'' of an edge, i.e. the average number of rows/columns incident to the edge. It follows that to update $k$ eigenvectors following $r$ edge changes, the complexity is $\mathcal{O}(rk(\bar{N}^2 n + \bar{N}^3))$. In contrast, the approach presented in our paper has a smaller update cost for a set of vertex or edge weight changes between $\Lm_t$ and $\Lm_{t+1}$, since changes are considered in a batch fashion, as will later become clear.  A further problem with the \textit{Ning et al.} approach is that eigenvectors are updated independently of one another and hence one loses the orthogonality $\vv_i^T\Dm\vv_j = \delta(i, j)$, where $\delta$ is the Kronecker delta function (taking the value $1$ if $i=j$ and $0$ otherwise), and vectors can become correlated for example.  

Another way of improving the efficiency of spectral clustering is to compute approximate eigen-decompositions at each iteration, for example by using the Nystr\"{o}m approach \cite{williams01nystrom}. The Nystr\"{o}m method is used for spectral graph clustering in \cite{fowlkes2004spectral} in conjunction with image segmentation. To estimate the eigenvalues and eigenvectors of $\Am \in \mathbb{R}^{n \times n}$, one first finds a matrix $\Am[I, I]$ in which $I \in \{1, \ldots n\}^m$ is a set of indices selected uniformly at random. If we assume that $\Am[I, I]$ is positive definite, we can take its square root $\Am[I, I]^{1/2}$. One then defines 
\begin{displaymath} 
\Sm = \Am[I, I] + \Am[I, I]^{-1/2}\Am[I, \bar{I}]\Am[I, \bar{I}]^T\Am[I, I]^{-1/2}, 
\end{displaymath}
where $\bar{I}$ is the complement of $I$, and diagonalises $\Sm$ using the eigen-decomposition $\Sm = \Um\Lambdam\Um^T$.  Let $\Vm = \Am[:, I] \Am[I, I]^{-1/2} \Um \Lambdam^{-1/2}$, then it can be shown that the Nystr\"{o}m approximation of $\Am$, given by $\tilde{\Am}_I = \Am[:, I]\Am[I, I]^{-1}\Am[I, :]$, is equivalent to $\Vm \Lambdam\Vm^T$.  If $\Am$ is indefinite a more complicated two-step procedure is required, see \cite{fowlkes2004spectral} for details. The resulting approximation is applied to find the first few eigenvectors of the normalised Laplacian matrix at a total cost of $\mathcal{O}(nm^2 + m^3)$. Several efficiency improvements based on this approach are proposed in \cite{li2011time} which finds the largest $k$ approximate eigenvectors of $\Dm^{-\frac{1}{2}}\Wm\Dm^{-\frac{1}{2}}$ at a reduced time and space complexity.  

The quality of the resulting approximation is determined by the extent the submatrix $\Am[\bar{I}, \bar{I}]$ is spanned by $\Am[I, :]$. One would naturally expect a good approximation for example when $\Am[I, :]$ spans the space of $\Am$. The choice of sampling of $I$ also affects approximation quality, and empirical and theoretical work in \cite{kumar2009sampling} suggests random sampling without replacement over the non-uniform sampling in \cite{drineas2005nystrom, drineas2006fast}. Hence for our later empirical work, we will use uniform random sampling without replacement. In addition to \cite{kumar2009sampling}, error bounds for the Nystr\"{o}m method are presented in \cite{drineas2005nystrom, mahdavi2012improved} in terms of the matrix approximation error using the Frobenius or spectral norm  (given by $\|\Am\|_F = \sqrt{\tr(\Am^T\Am)}$ and $\|\Am \|_2 = \sqrt{\lambda_{max}(\Am^\Tm\Am)}$ respectively). One of the disadvantages of applying the Nystr\"{o}m based approaches of \cite{fowlkes2004spectral, li2011time} on graphs is that by sampling only a subset of columns of the normalised Laplacian one can exclude important edges which help define clusters. These approaches are naturally more effective when a set of points $\{\xv_1, \ldots, \xv_n\} \in \mathbb{R}^d$ is mapped into a weighted graph using for example the Gaussian weighted distance $\Wm_{i, j} = \exp(-\|\xv_i - \xv_j\|^2)/2\sigma^2$, for $\sigma \in \mathbb{R}^+$, which is very informative about the relative positions of the points. 

Another class of algorithms which are potentially useful for spectral clustering is the randomised SVD \cite{halko2011finding}.  Notice that the SVD of a symmetric positive semi-definite matrix is identical to its eigendecomposition and hence SVD algorithms can be applied to spectral clustering. We recount an algorithm from \cite{halko2011finding} which is used in conjunction with kernel Principal Components Analysis (KPCA, \cite{sch98nonlinear}) in \cite{yun2011nystrom}. Algorithm \ref{alg:randomSVD} provides the associated pseudo-code. The purpose of the first three steps is to find an orthogonal matrix $\Qm$ such that the projection of $\Am$ onto $\Qm$ is a good approximation of $\Ym$, whose columns are random samples from the range of $\Am$, under a rank-$k$ projection. This projection is then used in conjunction with $\Qm$ so that one need only find the SVD of the smaller matrix $\Bm$. The complexity of this approach is $\mathcal{O}((qr+r^2)(m + n))$, which for a square matrix of size $n$ collapses simply to $\mathcal{O}((qr+r^2)n)$

\begin{algorithm}
\caption{Randomised SVD \cite{halko2011finding}}
\begin{algorithmic}[1]
\REQUIRE Matrix $\Am \in \mathbb{R}^{m \times n}$,  number of projection vectors $r$, exponent $q$ 
\STATE Generate a random Gaussian matrix $\Omegam \in \mathbb{R}^{n \times r}$
\STATE Create $\Ym = (\Am\Am^T)^q \Am \Omegam$ by alternative multiplication with $\Am$ and $\Am^T$ \label{step:power} 
\STATE Compute $\Ym = \Qm\Rm$ using the QR-decomposition 
\STATE Form $\Bm = \Qm^T\Am$ and compute SVD $\Bm = \hat{\Um}\Sigmam\Vm^T$
\STATE Set $\Um = \Qm\hat{\Um}$
\RETURN Approximate SVD $\Am \approx \Um\Sigmam\Vm^T$ 
\end{algorithmic}\label{alg:randomSVD}
\end{algorithm}


\section{SVD Updating} 
\label{sec:svd_update}
An important aspect of the incremental clustering method described later lies in its ability to efficiently compute the eigenvectors of a matrix from the eigenvectors of a submatrix, and for clarity's sake, we outline the SVD-updating algorithm of \cite{zha1999updating}. The SVD of $\Am \in \mathbb{R}^{m \times n}$ is the decomposition 
$$
\Am=\Pm \Sigmam \Qm^T,
$$
where $\Pm = [\pv_1, \ldots, \pv_m]$, $\Qm = [\qv_1, \ldots, \qv_n]$ and $\Sigmam = \diag(\sigma_1, \ldots, \sigma_r)$ are respectively the orthogonal matrices of left and right singular vectors and a diagonal matrix of singular values $\sigma_1 \geq \sigma_2 \geq, \ldots, \geq \sigma_r$, with $r = \min(m, n)$. 

\par In \cite{zha1999updating}, the authors use the SVD of $\Am$ to approximate the SVD of $[\Am \quad \Bm]$, with  $\Bm \in \mathbb{R}^{m \times p}$, without recomputing the SVD of the new matrix.  It is known that the best $k$-rank approximation of $\Am$, using the Frobenius norm error, is given by its SVD, namely 
$$
\Am_k = \Pm_k\Sigmam_k\Qm_k^T,
$$
where $\Pm_k$, $\Qm_k$, $\Sigmam_k$ correspond to the $k$ largest singular values. The general idea of the algorithm is to write
$$
[\Am_k \quad \Bm] = \Sm\Lambdam\Rm^T,
$$
in which $\Sm$ and $\Rm$ are  matrices with orthonormal columns. One then takes the rank-$k$ SVD $\Lambdam = \Gm_k \Gammam_k \Hm_k^T$ and $(\Sm\Gm_k) \Gammam_k (\Rm\Hm_k)^T$ is the rank-$k$ approximation for $[\Am \quad \Bm]$. The dimensionality of $\Lambdam$ is generally much smaller than that of $[\Am_k \quad \Bm]$ and hence the corresponding SVD approximation is computationally inexpensive. 

This approach is more accurate than those presented in \cite{berry1995using, o1994information}, however it comes at additional computational cost. Furthermore, the authors prove that when the matrix $[\Am \quad \Bm]^T[\Am \quad \Bm]$ has the form $\Xm + \alpha^2 \Imat$ in which $\Xm$ is symmetric positive semi-definite with rank-$k$ then the rank-$k$ approximation of $[\Am \quad \Bm]$ is identical to that of $[\Am_k \quad \Bm]$. 

\section{Incremental Eigen-approximation} \label{sec:our_work} 

In this section we address three types of updating problem upon the largest $k$ eigenvectors of a symmetric matrix. The updates are general operations, however they will be explained in the context of spectral clustering later in Section \ref{sec:incrementalCluster}. Assuming that $\Ym_1, \Ym_2 \in \mathbb{R}^{n \times p}$ and one does not have direct access to $\Am \in \mathbb{R}^{m \times n}$ and $\Bm \in \mathbb{R}^{m \times p}$ but only to the matrices $\Cm = \Am^T\Am$, $\Am^T\Bm$ and $\Bm^T\Bm$, the updates are: 

\begin{enumerate} 
  \item Addition of a low-rank symmetric matrix $\Cm \rightarrow \Cm + \Um$ where $\Um = \Ym_1\Ym_2^T + \Ym_2\Ym_1^T$
  \item Addition of rows and columns  $\Am^T\Am \rightarrow [\Am \quad \Bm]^T[\Am \quad \Bm]$
  \item Removal of rows and columns $[\Am \quad \Bm]^T[\Am \quad \Bm] \rightarrow \Am^T\Am$
\end{enumerate}

We have written the updates above in terms of a symmetric matrix $\Am^T\Am$ to improve notation. Note that any positive semi-definite symmetric matrix $\Mm$ can be decomposed into the form $\Mm = \Am^T\Am$, where $\Am$ has real entries, for example by using an eigen-decomposition or incomplete Cholesky factorisation. 

\subsection{Addition of a Low-rank Symmetric Matrix} \label{subsec:eigenAdd}  

The first type of eigen-approximation we are interested in is the addition of a low-rank symmetric matrix. One computes the eigen-decomposition of $\Cm$ and then approximates the rank-$k$ decomposition of $\Cm_k + \Um$ where $\Um = \Ym_1\Ym_2^T + \Ym_2\Ym_1^T$ and $\Cm_k$ is the approximation of $\Cm$ using the $k$ largest eigenvectors (also known as the best $k$-rank approximation of $\Cm$). A similar but not applicable update is considered for the SVD case in \cite{zha1999updating} in which the rank-$k$ approximation of $\Am_k + \hat{\Ym}_1\hat{\Ym}_2^T$ is found from $\Am_k$ in which $\hat{\Ym}_1 \in \mathbb{R}^{m \times j}$ and $\hat{\Ym}_2 \in \mathbb{R}^{n \times j}$ for some $j$. The general idea in our case is to find a matrix with orthonormal columns $\tilde{\Qm}$ such that $\Cm_k + \Um = \tilde{\Qm} \Deltam \tilde{\Qm}^T$ for a square matrix $\Deltam$. To this purpose, we first project the columns of $\Ym_1$ into the space orthogonal to the $k$ largest eigenvectors $\Qm_k$ of $\Cm$ (note the deviation from standard notation), a process known as \textit{deflation}. Assuming that eigenvectors have unit norm, the matrix $\Ym_1$ is thus deflated as follows: 

$$
\bar{\Ym}_1 = (\Imat - \Qm_k\Qm_k^T)\Ym_1,
$$ 
at a cost of $\mathcal{O}(npk)$. Note that $\bar{\Ym}_1\Thetam_1$ for some $\Thetam_1$, is orthogonal to $\Qm_k$ since $\Qm_k^T\bar{\Ym}_1\Thetam_1 = (\Qm_k^T\Ym_1 - \Qm_k^T\Ym_1)\Thetam_1 = \textbf{0}$. If we take the SVD $\bar{\Ym}_1 = \bar{\Pm}_1\bar{\Sigmam}_1 \bar{\Qm}^T_1$ then $\bar{\Pm}_1$ is orthogonal to $\Qm_k$ since $\bar{\Pm}_1 = \bar{\Ym}_1\bar{\Qm}_1\bar{\Sigmam}^{-T}_1$ assuming $\bar{\Sigmam}_1$ has nonzero diagonal entries. 

At the next stage we would like to orthogonalise the columns of $\Ym_2$ with respect to both $\Qm_k$ and $\bar{\Pm}_1$. Hence we deflate $\Ym_2$ in the following way:  

$$
\bar{\Ym}_2 = (\Imat - \bar{\Pm}_1\bar{\Pm}_1^T - \Qm_k\Qm_k^T)\Ym_2,
$$ 
at cost $\mathcal{O}(npk)$, where we have used the fact that $\bar{\Pm}_1$ is orthogonal to $\Qm_k$. Proved in a similar way to the step used earlier, the matrix  in the column space of $\bar{\Ym}_2$, $\bar{\Ym}_2\Thetam_2$ for some $\Thetam_2$, is orthogonal to $\Qm_k$ and $\bar{\Pm}_1$. Hence, we compute the SVD $\bar{\Ym}_2 = \bar{\Pm}_2\bar{\Sigmam}_2 \bar{\Qm}^T_2$ and note that the matrices $\bar{\Pm}_1$, $\bar{\Pm}_2$ and $\Qm_k$ are mutually orthogonal and span the space spanned by $\Cm_k + \Um$. This allows one to write 

\begin{displaymath} 
\Cm_k + \Um = \tilde{\Qm} \Deltam \tilde{\Qm}^T
\end{displaymath} 
as required with $\tilde{\Qm} = [\Qm_k \; \bar{\Pm}_1 \; \bar{\Pm}_2] \in \mathbb{R}^{n \times (k+2p)}$ and $\Deltam = \tilde{\Qm}^T (\Cm_k + \Um) \tilde{\Qm}$, or equivalently 

\begin{displaymath} 
\Deltam = 
\left[ \begin{array}{c c c} \Omegam_k + \Qm_k^T\Um\Qm_k & \Qm_k^T\Um\bar{\Pm}_1& \Qm_k^T\Ym_1\bar{\Qm}_2\bar{\Sigmam}_2 \\ 
\bar{\Pm}_1^T\Um\Qm_k & \bar{\Pm}_1^T\Um\bar{\Pm}_1 & \bar{\Sigmam}_1\bar{\Qm}_1^T\bar{\Qm}_2\bar{\Sigmam}_2 \\
\bar{\Sigmam_2}\bar{\Qm}_2^T\Ym_1^T\Qm_k &  \bar{\Sigmam}_2\bar{\Qm}_2^T\bar{\Qm}_1\bar{\Sigmam}_1 & \textbf{0} \end{array} \right], 
\end{displaymath}
in which $\Deltam \in \mathbb{R}^{(k+2p) \times (k+2p)}$. We take the rank-$k$ eigen-decomposition $\Deltam_k = \Hm_k\Pim_k\Hm_k^T$ and then the final eigen-approximation is given by $(\tilde{\Qm} \Hm_k) \Pim_k  (\tilde{\Qm} \Hm_k)^T$ in which it is easy to verify that the columns of $\tilde{\Qm} \Hm_k$ are orthonormal. 

The deflation and SVD of $\bar{\Ym}_1$ and $\bar{\Ym}_2$ cost $\mathcal{O}(npk)$ and  $\mathcal{O}(np^2)$ respectively and the eigen-decomposition of $\Deltam$ is $\mathcal{O}((k+2p)^3)$. In order to compute $\Deltam$ one can reuse the computations of $\Qm_k^T\Ym_1$, $\Qm_k^T\Ym_2$, $\bar{\Pm}_1^T\Ym_2$ which are used for deflations and also the matrices used for the SVD decompositions. Thus $\Deltam$ is found in $\mathcal{O}(p^3 + p^2k + pk^2)$, and the overall complexity of this algorithm is $\mathcal{O}((k^2+p^2)(p+k) + np(p + k))$. Of note here is that $n$ scales the complexity in a linear fashion, however costs are cubically related to $k$ and $p$. 

\subsection{Addition of Rows and Columns} \label{subsec:eigenConcat} 

In  correspondence with the SVD-updating method given above we consider the case in which one has the eigen-decomposition of $\Cm$ and then wants to find the rank-$k$ approximation of $\Em= [\Am \quad \Bm]^T[\Am \quad \Bm]$. Such a process is useful not just in incremental clustering but also in incrementally solving kernel Principal Components Analysis (KPCA,  \cite{sch98nonlinear}) for example. This update can be written in terms of that described above. This is seen by writing the former update  $\Am^T\Am \rightarrow [\Am \quad \Bm]^T[\Am \quad \Bm]$ in terms of the latter: 

\begin{displaymath} 
 \left[\begin{array}{c c}  
  \Am^T\Am  & \textbf{0} \\
  \textbf{0}  & \textbf{0} \\
 \end{array}\right] \rightarrow 
 \left[\begin{array}{c c}  
  \Am^T\Am  & \textbf{0} \\
  \textbf{0}  & \textbf{0} \\
 \end{array}\right] + 
 \left[\begin{array}{c c}  
  \textbf{0}  & \Am^T\Bm  \\
  \Bm^T\Am & \Bm^T\Bm  \\
 \end{array}\right].
\end{displaymath}

The second term on the right-hand side can be written as $\Ym_1\Ym_2^T + \Ym_2\Ym_1^T $ where $\Ym_1 = [\textbf{0} \quad \Imat_p]^T$ and  $\Ym_2 = [\Bm^T\Am \quad \frac{1}{2}\Bm^T\Bm]^T$. The eigenvectors of the first matrix on the right-hand side are found from those of $\Am^T\Am$ by simply adding $p$ zero rows to the existing eigenvectors, and the corresponding eigenvalues are identical. Additional eigenvectors are standard unit vectors spanning the $p$ new rows with corresponding eigenvalues as zero. A useful insight is that the deflated matrix $\bar{\Ym}_1 = \Ym_1$ and hence its SVD decomposition can be written directly as $\bar{\Ym}_1 = [\textbf{0} \quad \Imat_p]^T \Imat_p \Imat_p$.

\subsubsection{Alternative Approach} \label{subsec:alternative_removing}
Here we outline a simpler and more direct approach for the addition of rows and columns to a matrix. First let $\Cm = \Qm \Omegam \Qm^T$ in which $\Qm$ is a matrix of eigenvectors and $\Omegam$ is a diagonal matrix of eigenvalues.  Note that $\hat{\Em} = [\Am_k \quad \Bm]^T[\Am_k \quad \Bm]$ can be written as $\tilde{\Qm} \Deltam \tilde{\Qm}^T$ for a square matrix $\Deltam$. In our case we have 

\begin{eqnarray*} 
\tilde{\Qm} = \left[ \begin{array}{c c}  \Qm_k & \textbf{0} \\ \textbf{0} & \Imat_p
\end{array} \right] \mbox{ and } \Deltam = \left[\begin{array}{c c} \Omegam_k &
\Qm_k^T \Am_k^T\Bm \\ \Bm^T\Am_k\Qm_k & \Bm^T\Bm   \\  \end{array} \right],
\end{eqnarray*}
noting that $\Am_k\Qm_k\Qm_k^T = \Qm_k\Omegam_k\Qm^T_k\Qm_k\Qm_k^T = \Am_k$ since $\Qm^T_k\Qm_k = \Imat$. Furthermore, note that $\Qm_k^T\Am^T\Bm = \Qm_k^T\Qm\Sigmam\Pm^T\Bm = \Qm_k^T\Qm_k\Sigmam_k\Pm_k^T\Bm = \Qm_k^T\Am_k^T\Bm$ and calculating $\Qm_k^T\Am_k^T\Bm$ is $\mathcal{O}(npk)$. It follows that $\Deltam \in \mathbb{R}^{(k+p) \times (k+p)}$ can be found using $\Qm_k$, $\Am^T\Am$, and $\Am^T\Bm$ and $\Bm^T\Bm$. In the final step, and analogous to the SVD case we take the rank-$k$ eigen-decomposition  $\Deltam = \Hm_k\Pim_k\Hm_k^T$ at a cost of $\mathcal{O}((k+p)^3)$ and then the rank-$k$ eigen approximation of $\hat{\Em}$ is given by $(\tilde{\Qm} \Hm_k) \Pim_k (\tilde{\Qm} \Hm_k)^T$. 

Notice that this eigen-update almost follows directly from the SVD update in \cite{zha1999updating} however one need not use the QR decomposition of the deflated $\Bm$. Furthermore, the eigen-approximation above is identical to that at the start of Section \ref{subsec:eigenConcat} with $\Ym_1 = [\textbf{0} \quad \Imat_p]^T$ and  $\Ym_2 = [\Bm^T\Am_k \quad \frac{1}{2}\Bm^T\Bm]^T$. The difference between the method above and that of Section \ref{subsec:eigenConcat} is the former uses $\Am_k$ as opposed to $\Am$ in $\Ym_2$ which results in a greater error. 

\subsection{Removing Rows and Columns} 

Observe that removing rows and columns is equivalent to zeroing the corresponding rows/columns:  

\begin{displaymath} 
 \hat{\Cm} = \left[\begin{array}{c c}  
  \Am^T\Am   & \Am^T\Bm  \\
  \Bm^T\Am & \Bm^T\Bm  \\
 \end{array}\right]
\rightarrow 
 \left[\begin{array}{c c}  
  \Am^T\Am   & \Am^T\Bm  \\
  \Bm^T\Am & \Bm^T\Bm  \\
 \end{array}\right] - 
 \left[\begin{array}{c c}  
  \textbf{0}  & \Am^T\Bm  \\
  \Bm^T\Am & \Bm^T\Bm  \\
 \end{array}\right],
\end{displaymath}
and in this form one can see the connection to Section \ref{subsec:eigenConcat}. Again, one can write the second term on the right-hand side as $\Ym_1\Ym_2^T + \Ym_2\Ym_1^T $ where $\Ym_1 = [\textbf{0} \quad \Imat_p]^T$ and  $\Ym_2 = -[\hat{\Bm}^T\hat{\Am} \quad \frac{1}{2}\hat{\Bm}^T\hat{\Bm}]^T$ where $\hat{\Bm}^T\hat{\Am}$ and $\hat{\Bm}^T\hat{\Bm}$ are found using the rank-$k$ approximation of $\hat{\Cm}$. Since we are updating the rank-$k$ approximation of $\hat{\Cm}$, the final eigen-approximation will have zero elements in the eigenvectors at rows corresponding to those row/columns that are deleted.

\subsection{Discussion} \label{subsec:update_interpretation}

A similar eigen-update to that of Section \ref{subsec:alternative_removing} is considered in \cite{kwok03incrementalEig} and used in conjunction with PCA \cite{hotelling33pca} and KPCA in \cite{zhao2006novel}. In PCA, one starts with a set of examples $T = \{\xv_1, \xv_2, \ldots, \xv_n\}$ where $\xv_i \in \mathbb{R}^d$. These examples are centred, and then one finds the $k$ largest eigenvectors $\uv_1, \ldots, \uv_k$ of the covariance matrix $\Cm = 1/n \sum_{i=1}^n \tilde{\xv}_i\tilde{\xv}_i^T$, where $\tilde{\xv}$ is a centred example. KPCA functions similarly, except that one finds the largest eigenvectors of a kernel matrix $\Km \in \mathbb{R}^{n \times n}$ which is computed using the centred examples. In the incremental setting one approximates the eigenvectors of the covariance or kernel matrix on the addition of a new set of examples. In \cite{zhao2006novel} the authors phrase the problem as a series of SVD updates, and we use a more direct approach in Section \ref{subsec:alternative_removing}.  Furthermore, whereas \cite{zhao2006novel} examines the addition of rows and columns to a positive semi-definite matrix, our work is more general in that we additionally consider removal of rows and columns and addition of a low-rank symmetric matrix. One key novelty of this paper is the update of Section \ref{subsec:eigenAdd} which encapsulates all three of these updates and yet does not simply follow from the SVD updating work of \cite{zha1999updating}. 

To formalise this notion, the updates outlined above can be written in terms of the addition of a low-rank symmetric matrix $\Um$ to a positive semi-definite matrix $\Cm$. Our approximation method computes $\Cm_k + \Um$ via its expression as $\tilde{\Qm} \Hm_k\Pim_k\Hm_k^T \tilde{\Qm}^T$ where $\tilde{\Qm}$ is a matrix with orthogonal columns spanning $\Cm_k + \Um$ and  $\Hm_k$  and $\Pim_k$ represent the largest $k$ eigenvector and eigenvalues of a matrix $\Deltam=\tilde{\Qm}^T \left(\Cm_k + \Um\right) \tilde{\Qm}$. The following lemma shows the consequence of this approach. 

\begin{lemma} 
Decompose a matrix $\Zm = \Fm\Lambdam\Fm^T$ where $\Fm$ is \emph{any}  matrix with orthonormal columns $\Fm^T\Fm = \Imat$. For some $k$ find the best rank-$k$ eigen-approximation $\Lambdam_k = \Hm_k\Pim_k\Hm_k^T$ in which $\Hm_k$ and $\Pim_k$ are the largest $k$ eigenvectors and eigenvalues of $\Lambdam$, and let $\hat{\Zm} = \Fm\Hm_k\Pim_k\Hm_k^T\Fm^T$. Then the best rank-$k$ eigen-decomposition of $\Zm$ is given by:
$$
\Zm_k = \Um_k\Sm_k\Um_k^T = \hat{\Zm}.$$ 
\label{lem:update}\end{lemma}
\begin{proof} 
Note that $\Lambdam = \Fm^T\Zm\Fm$ due to the orthogonality of $\Fm$. Let $\uv$, $s$ be the eigenvectors and eigenvalues of $\Zm$, and define $\uv = \Fm\vv$ for some $\vv$ then  $\Fm^T\Zm\Fm\vv = s\vv$. This implies that the eigenvalues of $\Lambdam$ are the same as those of $\Zm$ and the eigenvectors are related by $\Um = \Fm\Vm$. Hence we have $\Vm = \Hm$ and $\Sm = \Pim$ which implies $\hat{\Zm} = \Fm\Hm_k\Pim_k\Hm_k^T\Fm^T = \Um_k \Sm_k \Um_k^T = \Zm_k$ as required. The only condition on $\Fm$ is that the eigenvectors of $\Zm$ are in the column space of $\Fm$. This must be the case however since $\Zm = \Um\Sm\Um^T = \Fm\Vm\Sm\Vm^T\Fm^T$. 
\end{proof}

Hence, the update  $\Cm_k + \Um$ of Section \ref{subsec:eigenAdd} is identical to the best rank-$k$ approximation of $\Cm_k + \Um$.

\section{Eigen-approximation Quality}\label{sec:quality}  
This section aims to bound the quality of the proposed eigen-approximation approach. As mentioned in Section \ref{subsec:update_interpretation}, the proposed approach replaces the expected best rank-$k$ approximation of a matrix $\Am + \Bm$ by the best rank-$k$ approximation of $\Am_k + \Bm$, where $\Am_k$ denotes the best rank-$k$ approximation of $\Am$. Hence, our objective is to control the difference between $(\Am + \Bm)_k$ and $(\Am_k + \Bm)_k$.

 The result in Section \ref{subsec:update_interpretation} gives us a first insight into the approximation error of the updates described: the residual matrix is that corresponding to the eigenvectors and eigenvalues after $k$. One can see that for any matrix $\Cm$

$$\| \Cm - \Cm_k \|_F^2 = \|\Cm_{k^\bot} \|_F^2 = \sum_{i=k+1}^n \omega_i^2,$$ 
where $\omega_i$ is the $i$th eigenvalue of $\Cm$ (unless otherwise stated eigenvalues are always given in descending order) and $\Cm_{k^\bot}$ is the approximation of $\Cm$ using eigenvectors/eigenvalues after $k$. This implies that $\Cm$ is well approximated by the largest $k$ eigenvectors if the sum of the square of the remaining eigenvalues is small. 

This certainly gives us insight into when our eigen-updating approach will be accurate, however the kind of matrices we will work with do not have this property in general. We now turn to matrix perturbation theory \cite{stewart1990matrix} in order to learn more about the approximated eigenvectors of the updated matrix. In a nutshell, it lies in controlling the \emph{angle} between two \emph{invariant subspaces}. Before giving the corresponding theorem, the following subsection introduces the necessary notions of invariant subspaces of a matrix and of the canonical angles between subspaces.

\subsection{Invariant Subspaces and Canonical Angles}

We begin by introducing the simple concept of an invariant subspace. 

\begin{definition} 
 The subspace $\mathcal{X}$ is an \emph{invariant subspace} of $\Am$ if $\Am \mathcal{X} \subset \mathcal{X}$. 
\end{definition}

It can also be shown that if the columns of $\Xm$ form a basis for $\mathcal{X}$ of $\Am$ then there is a unique matrix $\Lm$ such that $\Am\Xm = \Xm\Lm$.  The matrix $\Lm$ is a representation of $\Am$ with respect to the basis $\Xm$, and it has identical eigenvalues to $\Am$. A useful decomposition in perturbation theory is to reduce $\Am$ to a block diagonal form. Let $\Xm_1$ be an orthogonal matrix which spans the invariant subspace of $\Am$, $\mathcal{X}_1$, and assume that we have a matrix $\Ym_2$ such that $[\Xm_1 \Ym_2]$ is unitary and $\Ym_2$ spans the space orthogonal to $\mathcal{X}_1$, then this allows us to write 

\begin{equation}
[\Xm_1 \Ym_2]^T \Am [\Xm_1 \Ym_2] =\left[ \begin{array}{c c} \Lm_1 & \Hm \\ \mathbf{0} & \Lm_2 \end{array} \right],
\label{eqn:blockInvSub}
\end{equation}
in which $\Lm_1 = \Xm_1^T\Am\Xm_1$, $\Lm_2 = \Ym_2^T\Am\Ym_2$ and $\Hm = \Xm_1^T\Am\Ym_2$. The above equation is known as the \emph{reduced form} of $\Am$ with respect to $[\Xm_1 \Ym_2]$. The proof of why the bottom left block of this matrix is zero is straightforward, see \cite{stewart1990matrix} for details. This gives us the knowledge to define a \emph{simple invariant subspace}.

\begin{definition} 
Let $\mathcal{X}$ be an invariant subspace of $\Am$ and consider the reduced form of Equation \eqref{eqn:blockInvSub}, then $\mathcal{X}$ is a \emph{simple invariant subspace} of $\Am$ if there are no common eigenvalues between $\Lm_1$ and $\Lm_2$. 
\end{definition}

Notice that a simple invariant subspace has a complementary space, defined as follows.

\begin{definition} 
 Let the simple invariant subspace $\mathcal{X}_1$ have the reduced form of Equation \eqref{eqn:blockInvSub} with respect to the orthogonal matrix $[\Xm_1 \Ym_2]$. Then there exist $\Xm_2$ and $\Ym_1$ such that $[\Xm_1 \Xm_2]^{-1} = [\Ym_1 \Ym_2]^T$ and 
\begin{equation}
 \Am = \Xm_1\Lm_1\Ym_1^T + \Xm_2\Lm_2\Ym_2^T,
\end{equation}
where $\Lm_i = \Ym_i^T\Am\Xm_i$, $i=1,2$. This form of $\Am$ is known as the \emph{spectral resolution} of $\Am$ along $\mathcal{X}_1$ and $\mathcal{X}_2$. 
\end{definition}

This allows us to introduce a theorem essential to our main result. However, first we must define the notion of angle between two subspaces. 

\begin{theorem} 
Let $\Xm_1, \Ym_1 \in \mathbb{R}^{n \times \ell}$ with $\Xm_1^T\Xm_1 = \Imat$ and $\Ym_1^T\Ym_1 = \Imat$. If $2\ell \leq n$, there are unitary matrices $\Qm, \Um_{11}, \Vm_{11}$ such that 
\begin{displaymath} 
\Qm\Xm_1\Um_{11} = \left( \begin{array}{c} \Imat_\ell \\ \textbf{0} \\ \textbf{0}  \end{array} \right) 
\text{ and } 
\Qm\Ym_1\Vm_{11} = \left( \begin{array}{c} \Gammam \\  \Sigmam \\ \textbf{0} \end{array} \right), 
\end{displaymath}
in which $\Gammam  = \diag(\gamma_1, \ldots, \gamma_\ell)$ with $0 \leq \gamma_1 \leq \ldots \leq \gamma_\ell$, $\Sigmam = \diag(\sigma_1, \ldots, \sigma_\ell)$ with  $\sigma_1 \geq \cdots \geq \sigma_\ell \geq 0$, and $\gamma_i^2 + \sigma_i^2 = 1$, $i=1,\ldots,\ell$.  If $2\ell > n$ then $\Qm, \Um_{11}, \Vm_{11}$ can be chosen so that  
\begin{displaymath} 
\Qm\Xm_1\Um_{11} = \left( \begin{array}{c c} \Imat_{n-\ell} & \textbf{0} \\ \textbf{0} &  \Imat_{2\ell - n} \\ \textbf{0} & \textbf{0}  \end{array} \right)
\text{ and } 
\Qm\Ym_1\Vm_{11} = \left( \begin{array}{c c} \Gammam & \textbf{0}  \\  \textbf{0} & \Imat_{2\ell - n} \\ \Sigmam  & \textbf{0} \end{array} \right), 
\end{displaymath}
in which $\Gammam  = \diag(\gamma_1, \ldots, \gamma_{n-\ell})$ with $0 \leq \gamma_1 \leq \ldots \leq \gamma_{n-\ell}$, $\Sigmam = \diag(\sigma_1, \ldots, \sigma_{n-\ell})$ with  $\sigma_1 \geq \cdots \geq \sigma_{n-\ell} \geq 0$, and $\gamma_i^2 + \sigma_i^2 = 1$, $i=1,\ldots,{n-\ell}$. \label{thm:ssBases}
\end{theorem}

Geometrically, let $\mathcal{X}_1$ and $\mathcal{Y}_1$ be subspaces  of dimension $\ell$ and $\mathcal{Q}$ be a unitary transformation. Then the matrices $\Qm\Xm_1\Um_{11}$ and $\Qm\Ym_1\Vm_{11}$, with $\Xm_1 \in \mathcal{X}_1$, $\Xm_2 \in \mathcal{X}_2$, $\Qm \in \mathcal{Q}$ ,  form bases of $\mathcal{QX}_1$ and $\mathcal{QY}_1$ and $\sigma_i$ and $\gamma_i$ can be regarded as sines and cosines of the angles between the bases. We now define a measure of similarity between subspaces. 

\begin{definition} 
 Let $\mathcal{X}$ and $\mathcal{Y}$ be subspaces of the same dimension, then the \emph{canonical angles} between the subspaces are the diagonal entries of the matrix $\sin^{-1} \Sigmam=\diag(\sin^{-1}(\sigma_1), \ldots, \sin^{-1}(\sigma_{n-\ell}))$ where $\Sigmam$ is the matrix defined in Theorem \ref{thm:ssBases}.
\end{definition}

It follows that $\Sigmam$ is a measure of how two subspaces differ.

\subsection{Angle Between Exact and Updated rank-$k$ Approximations}
We are now ready to state the main theorem required for our result. 

\begin{theorem}\label{thm:canAngle} 
Let $\Am$ be a Hermitian matrix with spectral resolution given by $[\Xm_1 \Xm_2]^T \Am [\Xm_1 \Xm_2] = \diag(\Lm_1, \Lm_2)$ where $[\Xm_1 \Xm_2]$ is unitary. Let $\Zm \in \mathbb{R}^{n \times k}$ have orthonormal columns, $\Mm$ be any Hermitian matrix of order $k$ and define the residual matrix as $\Rm = \Am\Zm - \Zm\Mm$. Let $\lambda(\Am)$ represent the set of eigenvalues of $\Am$ and suppose that $\lambda(\Mm) \subset [\alpha, \beta]$ and for some $\delta > 0$, $\lambda(\Lm_2) \subset \mathbb{R} \setminus [\alpha - \delta, \beta + \delta]$. Then, for any unitary invariant norm, we have: 
\begin{equation}
 \|\sin \Theta(\mathcal{R}(\Xm_1), \mathcal{R}(\Zm))  \| \leq \frac{\|\Rm\|}{\delta},
\end{equation}
where $\mathcal{R}(\cdot)$ is the column space of a matrix. 
\end{theorem}

Before we introduce the main result we present a result by Weyl \cite{weyl1912asymptotische} which characterises the perturbation in the eigenvalues of a matrix.

\begin{theorem}[Weyl, \cite{weyl1912asymptotische}] \label{thm:weyl} 
Define $\Am \in \mathbb{R}^{n \times n}$  and let $\tilde{\Am} = \Am + \Em$ be its perturbation. The eigenvalues of $\Am$ and $\Em$ are given by $\lambda_i$ and  $\epsilon_i$, $i=1,\ldots, n$, respectively. Then the eigenvalues of $\tilde{\Am}$ are, for $i=1, \ldots, n$, $\tilde{\lambda}_i \in [\lambda_i + \epsilon_n, \lambda_i + \epsilon_1]$. 
\end{theorem}

We can now present our main result which is closely related to the Davis-Kahan theorem \cite{davis1970rotation}. 

\begin{theorem} 
Consider a positive semi-definite matrix $\Am \in \mathbb{R}^{n \times n}$ with eigenvalues $\omega_1, \ldots, \omega_n$ and corresponding eigenvectors $\Qm = [\qv_1, \ldots, \qv_n]$. Let $\Bm \in \mathbb{R}^{n \times n}$ be symmetric with eigenvalues $\epsilon_i$ and $(\Am + \Bm)$ be positive semi-definite with eigen-decomposition $\Um \Gammam \Um^T$ where $\gamma_i$ are eigenvalues, $i=1, \ldots, n$. Fix integer $k$, let $\Am_k + \Bm$ have decomposition $\Vm\Pim\Vm^T$ with eigenvalues $\pi_1, \ldots, \pi_n$ and assume $\gamma_k \neq \gamma_{k+1}$. Then the following bounds hold on the canonical angles between the subspaces defined by $\Um_k$ and $\Vm_k$, assuming $\pi_k > \hat{\gamma}_{k+1}$, 

\begin{equation} 
 \|\sin \Theta(\mathcal{R}(\Um_k), \mathcal{R}(\Vm_k))  \|_F \leq \frac{\sqrt{\tr(\Vm_k^T\Am^2_{k^\bot}\Vm_k)}}{\pi_k - \hat{\gamma}_{k+1}},
\end{equation}
\begin{equation} 
 \|\sin \Theta(\mathcal{R}(\Um_k), \mathcal{R}(\Vm_k))  \|_2  \leq \frac{\sqrt{\lambda_{max}(\Vm_k^T\Am^2_{k^\bot}\Vm_k)}}{\pi_k - \hat{\gamma}_{k+1}},
\end{equation}
where $ \hat{\gamma}_{k+1} = \omega_{k+1} + \pi_{k+1}$. 
\label{thm:pertTheorem}\end{theorem} 
\begin{proof}
We will start by considering the first bound. It is clear that $\Um_k$ is a simple invariant subspace for $(\Am + \Bm)$. The spectral resolution of $\Am + \Bm$ is given by $[\Um_k \Um_{k^\bot}]$ since this matrix is unitary and we have $\Am+\Bm = \Um_k\Um_k^T(\Am +\Bm)\Um_k\Um_k^T + \Um_{k^\bot}\Um_{k^\bot}^T(\Am+\Bm)\Um_{k^\bot}\Um_{k^\bot}^T$. In the reduced form $\Lm_1 = \Gammam_k$ and $\Lm_2 = \Gammam_{k^\bot}$. Furthermore, we set $\Mm = \Pim_k$ and $\Zm = \Vm_k$. The residual matrix is given by $\Rm = (\Am+\Bm)\Vm_k - \Vm_k \Pim_k = \Am_{k^\bot}\Vm_k$ and 

\begin{eqnarray} 
 \|\Rm\|_F &=& \sqrt{\tr(\Vm_k^T\Am^2_{k^\bot}\Vm_k)}.  
\end{eqnarray}
We know that the eigenvalues of $\Mm$ fall within the range $[\pi_k, \pi_1]$ and those of $\Lm_2$ are bounded using Theorem \ref{thm:weyl} in the range $\gamma_i \in [\pi_i + \omega_n, \pi_i+\omega_{k+1}]$ for $i=k+1, \ldots, n$. Considering also the perturbation of eigenvalues of $\Am$ we can write $\gamma_i \leq \hat{\gamma}_i = \min(\omega_{k+1} + \pi_{k+1}, \omega_{k+1} + \epsilon_1) = \omega_{k+1} + \pi_{k+1}$ given $\pi_{k+1} \leq \omega_{k+1} + \epsilon_1$.  It follows that $\delta = \pi_k - \hat{\gamma}_{k+1}$ and plugging into Theorem \ref{thm:canAngle} gives the required result.  

The second bound is proved similarly except in this case we have 
\begin{eqnarray*} 
 \|\Rm\|_2 &=& \sqrt{\lambda_{max}(\Vm_k^T\Am^2_{k^\bot}\Vm_k)}.
\end{eqnarray*}

\end{proof}

Thus we have a bound on the angle between the subspace of the first $k$ eigenvectors of $\Am_k+ \Bm$ and the corresponding eigenvectors of it perturbation $\Am+\Bm$ without explicitly requiring the eigen-decomposition of $\Am+\Bm$. Provided the eigenvalues of $\Vm_k^T\Am^2_{k^\bot}\Vm_k$ are small and the eigengap $\pi_k - \hat{\gamma}_{k+1}$ is large one can be sure that the two subspaces have small canonical angles. One can see that under small perturbations $\Vm_k$ is close to $\Qm_k$ and hence $\Vm_k^T\Qm_{k^\bot}$ is small, resulting in tight bounds in the angles. In the case that the matrices involved correspond to normalised Laplacians this result corresponds well with similar results outlining a perturbation-based motivation of spectral cluster (see e.g. \cite{von2007tutorial}) which state that if the value of $\gamma_k - \gamma_{k+1}$ is large then one might reasonably expect a good clustering. The bound becomes loose when this eigengap is small, however in this case the clusters are less distinct even when computing the exact eigenvectors.     

\section{Incremental Cluster Membership} \label{sec:incrementalCluster}

We now return to the eigenproblem of Algorithm \ref{alg:normCluster}, $\tilde{\Lm}\vv = \lambda \vv$, in which we are interested in the eigenvectors with the smallest eigenvalues. Define the \emph{shifted Laplacian} as 

$$
\hat{\Lm} = 2\Imat - \tilde{\Lm} = \Imat + \Dm^{-\frac{1}{2}}\Wm\Dm^{-\frac{1}{2}},
$$
which is positive semi-definite since $\tilde{\Lm}$ is positive semi-definite with largest eigenvalue $2$. Note that by negating a matrix one negates the eigenvalues, leaving the eigenvectors the same, and similarly an addition of $\sigma\Imat$ increases the eigenvalues by $\sigma$ leaving the eigenvectors intact. Since we are interested in the smallest eigenvectors of $\tilde{\Lm}$ they correspond exactly to the maximum eigenvectors of $\hat{\Lm}$ and we can use the eigen-update methods described above. Observe that the shifted Laplacian is a normalised version of the \emph{signless Laplacian} \cite{haemers2004enumeration} defined as $\Lm^{+} = \Dm + \Wm$, and this can be seen from $\hat{\Lm} = \Dm^{-\frac{1}{2}}(\Dm + \Wm)\Dm^{-\frac{1}{2}} = \Dm^{-\frac{1}{2}}\Lm^{+}\Dm^{-\frac{1}{2}}$. 

Putting the ingredients together allows us to outline an efficient incremental method for performing graph clustering called Incremental Approximate Spectral Clustering (IASC), see Algorithm \ref{alg:incNormCluster}. At a high level the algorithm is quite simple: in the initialisation steps one computes the shifted Laplacian matrix for the first graph $\hat{\Lm}_1$ and then performs $k$-means clustering using the largest $k$ eigenvectors of this matrix. For the $t$-th successive graph, $t > 1$, we use the eigen-update methods above to approximate the largest eigenvectors of $\hat{\Lm}_t$ using the approximate eigenvectors computed at the previous iteration in step \ref{step:eigenConcat}. For these updates, it is simple to recover the matrices $\Ym_1$ and $\Ym_2$ given a change in edge weights. Note that we use the first method of Section  \ref{subsec:eigenConcat}  to compute eigenvector upon the addition of row and columns. Furthermore, we store the first $\ell$ eigenvectors of $\hat{\Lm}_1$, with $\ell 
\geq k$, and recompute eigenvectors every $T$ iterations in order to reduce cumulative errors  introduced in the loop at the expense of increased computation. When $\ell = n$ one recovers the exact eigenvectors at each iteration and Algorithms \ref{alg:incNormCluster} and \ref{alg:normCluster} become nearly equivalent. In the case that there is a significant eigengap between eigenvalues, one can fix both $\ell$ and $k$ according to Theorem \ref{thm:pertTheorem}.

\begin{algorithm}
\caption{Incremental Approximate Spectral Clustering}
\begin{algorithmic}[1]
\REQUIRE Graphs $G_1, \ldots, G_T$ of sizes $n_1, \ldots, n_T$, no. clusters $k$, approximation rank $\ell \geq k$, eigen-decomposition recomputation step $R$ 
\STATE Compute the shifted Laplacian for $G_1$, $\hat{\Lm}_1$ 
\STATE Find $\ell$ largest eigenvectors of $\hat{\Lm}_1$,  $\Vm_{\ell}^{(1)}$, let $\Vm_{k}^{(1)}$ be the first $k$ cols of $\Vm_{\ell}^{(1)}$
\STATE Normalise the eigenvector rows, $\Vm_{k}^{(1)} \leftarrow \diag(\Vm_{k}^{(1)}(\Vm_{k}^{(1)})^T)^{-\frac{1}{2}}\Vm_{k}^{(1)}$ 
\STATE Use $k$-means on rows of $\Vm_{k}^{(1)}$ and store indicators $\cv_1 \in \{1,\ldots,k\}^{n_1}$
\FOR{$t = 2 \to T$} 
\STATE Compute shifted Laplacian for $G_t$, $\hat{\Lm}_t$
\label{step:computeLaplacian}
\IF {$i\;\%\;R = 0$} 
  \STATE Recompute eigenvectors of $\hat{\Lm}_t$
\ELSE
\STATE Use rank-$\ell$ eigen-approximation of Section \ref{sec:our_work} \label{step:eigenConcat}
\ENDIF  
\STATE Normalise rows of $\Vm_{k}^{(t)}$, $\Vm_{k}^{(t)} \leftarrow \diag(\Vm_{k}^{(t)}(\Vm_{k}^{(t)})^T)^{-\frac{1}{2}}\Vm_{k}^{(t)}$ 
\STATE Use $k$-means on $\Vm_{k}^{(t)}$ with initial centroids $\cv_{t-1}$, store $\cv_t \in \{1,\ldots,k\}^{n_t}$
\ENDFOR
\RETURN Cluster membership $\cv_1 \in \{1,\ldots,k\}^{n_1}, \ldots, \cv_T \in \{1,\ldots,k\}^{n_T}$  
\end{algorithmic}\label{alg:incNormCluster}
\end{algorithm}

The complexity of Algorithm \ref{alg:incNormCluster} is dictated by the sparsity of the graphs and the extent of the change between successive graphs. We ignore the steps before the for-loop since in general they do not impact the overall complexity. At iteration $t$ and step \ref{step:computeLaplacian} one can compute $\hat{\Lm}_t$ from the weight and degree matrix at a cost of $\mathcal{O}(n_t + |E_t|)$. In the following step if there is a change between edges incident to vertices $S = \{v_{I_1}, \ldots ,v_{I_\ell}\}$ then the rows and columns corresponding to the union of the neighbours of $S$, $n(S)$, will change in the corresponding shifted Laplacian. In this case $p = |n(S)|$ in $\Ym_1, \Ym_2 \in \mathbb{R}^{n \times p}$ and if the neighbourhood of the vertices with changed edges $n(S)$ is small, this update can be efficiently computed as outlined in Section \ref{sec:our_work}. The cost of $k$-means is $\mathcal{O}(k^2ns)$ where $s$ is the number of iterations required for convergence. 

\section{Computational Results} 
\label{sec:exp}

In this section we study the clustering quality of five incremental strategies: the naive one which computes the exact eigenvectors at each iteration, denoted \texttt{Exact}, Ning et al.'s method (\texttt{Ning}) since it is a competing incremental strategy, \texttt{IASC},  \texttt{Nyst}, which uses the Nystr\"{o}m eigen-decomposition approximation of the shifted Laplacian, since it is often used in spectral clustering, and likewise the Randomised SVD method of Algorithm \ref{alg:randomSVD} denoted \texttt{RSVD}. First, however we observe the quality of the eigenvectors found using our eigen-updating approach.

\subsection{Quality of Approximate Eigenvectors}
\label{sec:eigen_quality}

We compare \texttt{Nystr\"{o}m} and \texttt{RSVD} approximations with our eigen-updating methods on a synthetic dataset generated in the following way: The initial graph contains 4 clusters of size 250 which are generated using an Erd\"{o}s-R\'{e}nyi \cite{erdos1959random} process with edge probability $p=0.1$. The Erd\"{o}s-R\'{e}nyi process creates edges independently randomly with a fixed probability $p$ for all pairs of vertices. For each successive graph we then add 50 random edges to simulate ``noise'' in the clusters for a total of 100 graphs. We then compute the largest $k=4$ eigenvectors of the shifted Laplacian using the full eigen-decomposition, the \texttt{Nystr\"{o}m} method, \texttt{RSVD} and the eigen-updating approach of Section \ref{sec:our_work}. For the \texttt{Nystr\"{o}m} method we use $m = 900$ randomly selected columns, for \texttt{RSVD} we use $r=\{100,~900\}$ random projections, and for the eigen-updating approach we update based on $\ell=\{4,~300\}$ approximate eigenvectors and eigenvalues found for the previous graph. The quality of the approximations is measured by $\|\sin \Theta(\mathcal{R}(\Um_k), \mathcal{R}(\Vm_k))\|_F$ where $\Um_k$ stands for the exact eigenvectors and $\Vm_k$ corresponds to the eigenvectors returned by one of the three considered eigen-approximation strategies. Applying Corollary I.5.4 of \cite{stewart1990matrix}, this norm is equal to $\|\Um_{k^\bot}\Vm_k\|_F$. The experiment is repeated with results averaged over 20 iterations. 

Observe that the quality of the approximation decreases with the noise amplitude, see Figure \ref{fig:nystromEigPlot}. For the three approximation strategies, the more eigenvectors/columns/random projection they consider, the smaller the canonical angles with the exact decomposition as one might expect. \texttt{Nystr\"{o}m} leads to a poor approximation whilst using 90\% of the columns. \texttt{RSVD} has poor results with 100 random projections, but is close to the exact decomposition with 900 random projections. The eigen-update strategy always does a better job than keeping the initial decomposition, and in this case there does not appear to be a large gain when moving from 300 eigenvectors to just 4. In the case in which we use the complete set of eigenvectors we would expect IASC to coincide with exact eigen-decomposition. 

\subsubsection{Perturbation Bound}

We now turn to demonstrating the effectiveness of the bound of Theorem \ref{thm:pertTheorem} on a synthetic dataset containing 150 vertices. The initial graph contains 3 clusters of size 50 which are generated using an Erd\"{o}s-R\'{e}nyi process with edge probability $p=0.3$, resulting in 1092 edges. For each successive graph we add 10 random edges and in total there is a sequence of 80 graphs. We then compute the bound of Theorem \ref{thm:pertTheorem} to measure the difference between the canonical angles of the real and approximated eigenvectors of the shifted Laplacian. We compare the results to the bound using the real eigenvalues of the Laplacian, i.e. $\delta = \pi_k - \gamma_{k+1}$. This process is repeated 50 times with different random seeds and the results are averaged. 

\begin{figure}[ht]
\begin{center}
\subfigure{\includegraphics[width=0.45\linewidth]{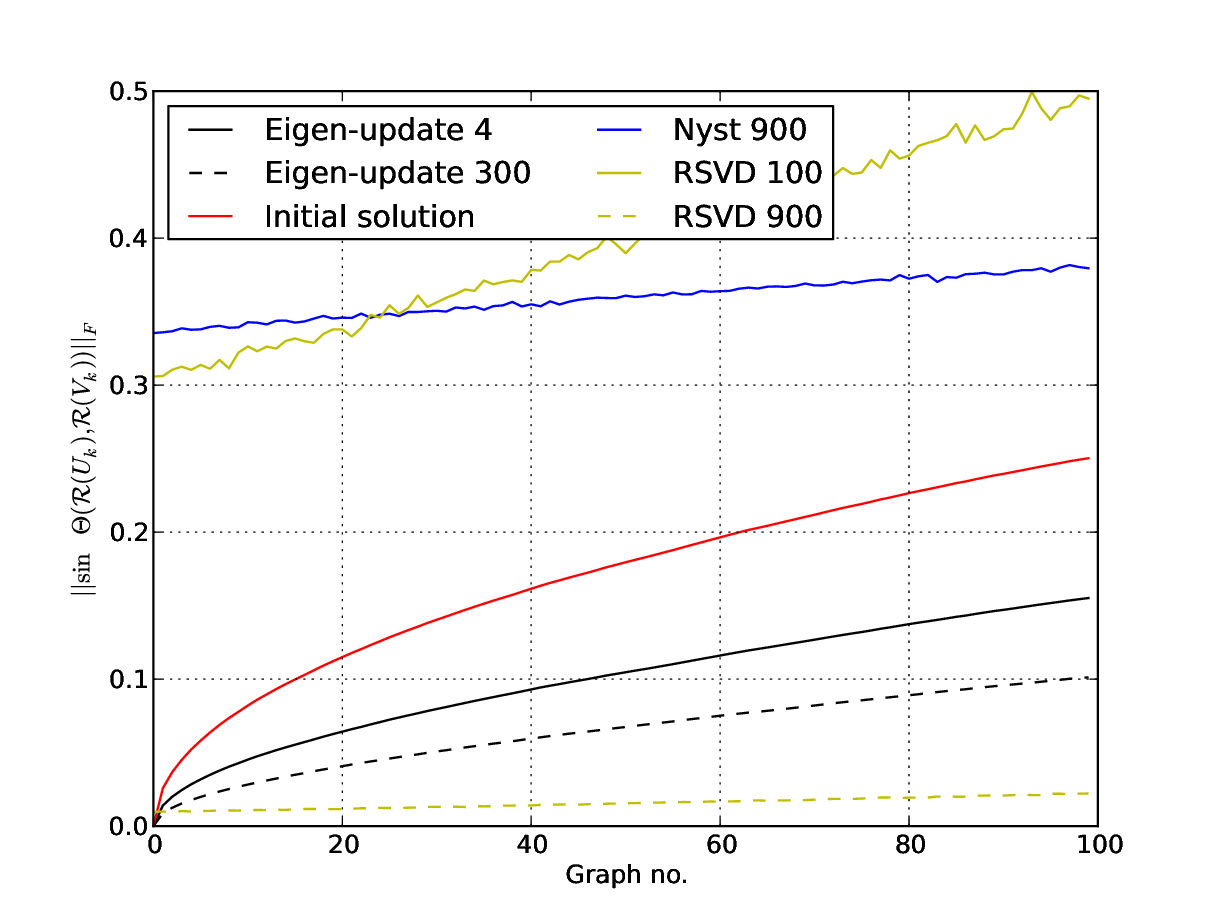}\label{fig:nystromEigPlot}} 
\subfigure{\includegraphics[width=0.45\linewidth]{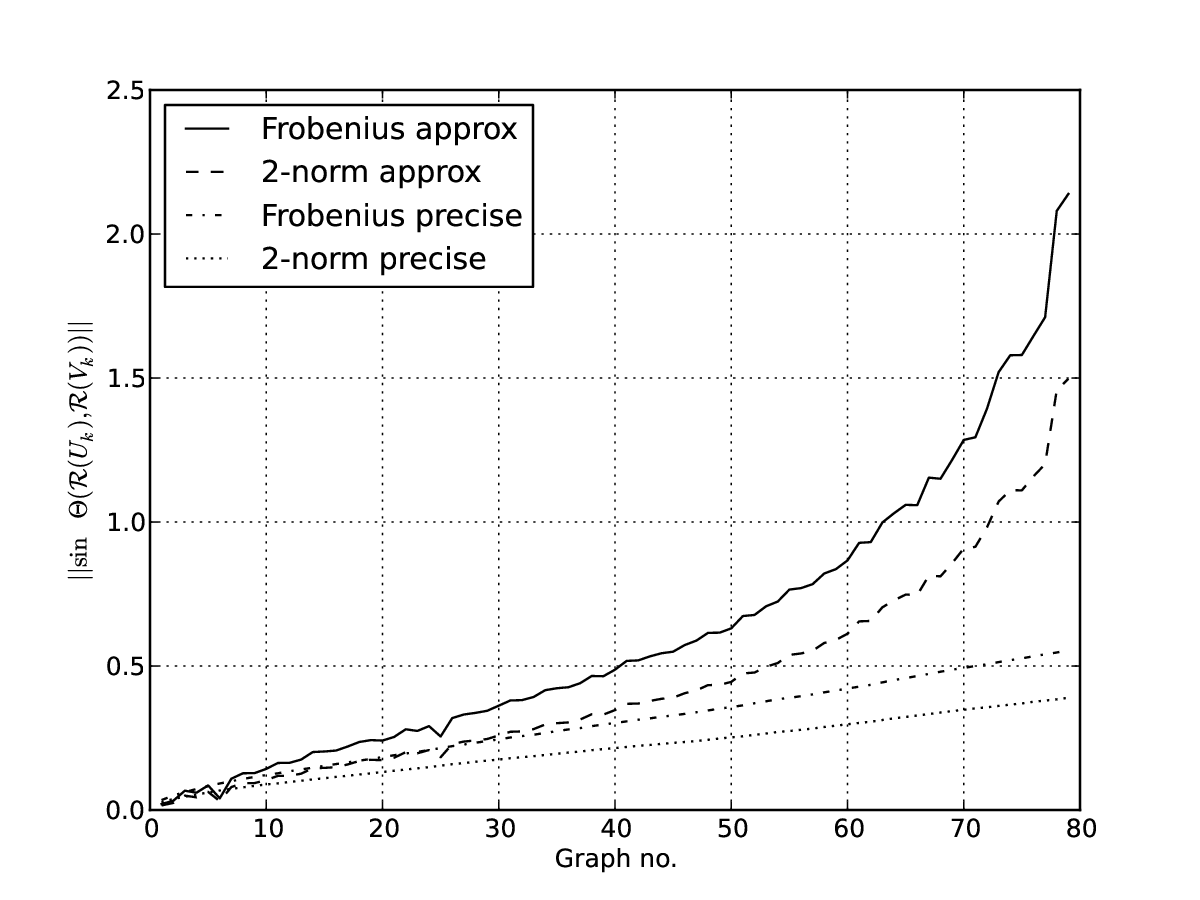}\label{fig:pertBoundPlot}}
\end{center} 
\caption{The left plot compares the canonical angles of the largest eigenvectors found using the approximation methods with the equivalent real eigenvectors. On the right we compare the perturbation bounds of Theorem \ref{thm:pertTheorem}. By ``precise'' we mean that we use $\delta = \pi_k - \gamma_{k+1}$ in the theorem.}
\end{figure}

Figure \ref{fig:pertBoundPlot} shows the resulting bounds for this sequence of graphs. The bound diverges slowly from the precise bound for many of the initial graphs, for example $\|\sin \Theta(\mathcal{R}(\Um_k), \mathcal{R}(\Vm_k))\|_F$ is bounded by $0.616\pm0.106$ versus $0.352\pm0.024$ at the 50th graph and the equivalent results for the 2-norm are $0.434\pm0.007$ and $0.248\pm0.02$. Soon after this point, we see a large divergence in the precise and approximate bounds, although the approximate bounds become trivial after the 78th graph for the Frobenius norm and the 73rd for the 2-norm corresponding to the addition of 770 and 730 edges respectively to the initial graph. When we look at the precise bound, one can see that at the last graph  $\|\sin \Theta(\mathcal{R}(\Um_k), \mathcal{R}(\Vm_k))\|_2 \leq 0.39$ despite nearly doubling the number of edges which points to the precision of the eigen-approximation in this case. 

\subsection{Clustering on Synthetic Data} 

To evaluate the clustering approaches, two separate synthetic datasets are considered, both of which are generated by an Erd\"{o}s-R\'{e}nyi process. The first dataset, called \texttt{3clust}, is based on a graph of 3 clusters of $60$ vertices each. This dataset allows us to compare the clustering quality as the clusters become more/less distinct and also the addition and removal of edges.   In the corresponding graph, any possible edge between two points from the same cluster occurs with probability $p_c = 0.3$ and the probability of edges between vertices in different clusters is selected from $p_g \in \{0.1, 0.2\}$. To generate a sequence of graphs we first allow only 20 vertices per cluster and then add 5 vertices to each cluster at a time until each one is of size 60. We then reverse the process so that the 10th graph is the same as the 8th, and the 11th is the same as the 7th etc., which allows us to test the clustering methods upon the removal of vertices. 

The second dataset, \texttt{discrepancy}, aims to examine the clustering methods on a more complex set of hierarchical clusters. Here, $180$ vertices are split into 3 clusters of equal size, and inside each cluster there are 3 sub-clusters. The initial graph is empty, and at each iteration $t\in\{1,\;\ldots,\;T\}$, $T=23$, an edge $\{v_i,v_j\}$ is added with a probability $p_\varepsilon+(p-p_\varepsilon)\ell t/T$, where $\ell$ is set to $1$ if $v_i$ and $v_j$ are in the same subgroup, $0.5$ if $v_i$ and $v_j$ are in the same group, and $0$ otherwise. In our case we set $p_\varepsilon = 0.0005$ and $p = 0.01$.  

For both datasets and the clustering approaches, $k$-means is run with $k$ corresponding to the number of clusters (3 and 9 for \texttt{3clust} and \texttt{discrepancy} respectively). For \texttt{IASC} we fix the number of eigenvectors $\ell \in \{3, 6, 12,  24\}$ with the \texttt{3clust} dataset and $\ell \in \{9, 72\}$ with \texttt{discrepancy} in order to test the approximation quality as this parameter varies. With \texttt{Nystr\"{o}m} we sample $m=90$ columns of the Laplacian matrix to find the approximate eigenvectors on \texttt{3clust} and select $m \in \{9, 72\}$ for \texttt{discrepancy}.
\texttt{RSVD} uses $r=24$ random projections on \texttt{3clust} and   $r \in \{9, 72\}$ for \texttt{discrepancy}.
On \texttt{discrepancy}, the approximation methods start with the 3rd graph in the sequence to allow the initial graph to contains enough edges. The experiments are repeated 50 times with different random graphs constructed using the methods described above, and the results are averaged. Clustering accuracy is measured through the \textit{Rand Index} \cite{Rand71} between the finest true clustering $\cal C$ and the learned one $\widehat{\cal C}$. Rand Index corresponds to the proportion of true answers to the question ``\emph{Are vertices $v_i$ and $v_j$ in the same cluster ?}''. More formally, it is given by
\[
 \operatorname{RandIndex}({\cal C},\widehat{\cal C}) =
\frac{\left|\left\{
v_i,v_j\in V :\; v_i \neq v_j,\;
\delta({\cal C}(v_i), {\cal C}(v_j))\neq \delta(\widehat{\cal C}(v_i),
\widehat{\cal C}(v_j)) \right\}\right|}{\left|\left\{v_i,v_j\in V :\; v_i \neq v_j\right\}\right|}
,
\]
where ${\cal C}(v_i)$ stands for the cluster index of vertex $v_i$ after clustering ${\cal C}$, $\delta$ is the Kronecker delta function and $| \mathcal{E}|$ denotes the cardinality of any finite set $\mathcal{E}$. The evaluation is completed by the computation of $\|\sin \Theta(\mathcal{R}(\Um_k), \mathcal{R}(\Vm_k))\|_F$, see Section \ref{sec:eigen_quality} for more details. The canonical angles for \texttt{Ning} are not given as \texttt{Ning} uses the random-walk Laplacian which (i) differs from the normalised Laplacian used with other approaches and (ii) is not symmetric, leading to non-orthogonal eigenvectors.

\newcommand{\subfigsize}{0.5}

\begin{figure}[ht]
\begin{center}
\subfigure[Rand Index, $p=0.1$]{\includegraphics[width=\subfigsize\linewidth]{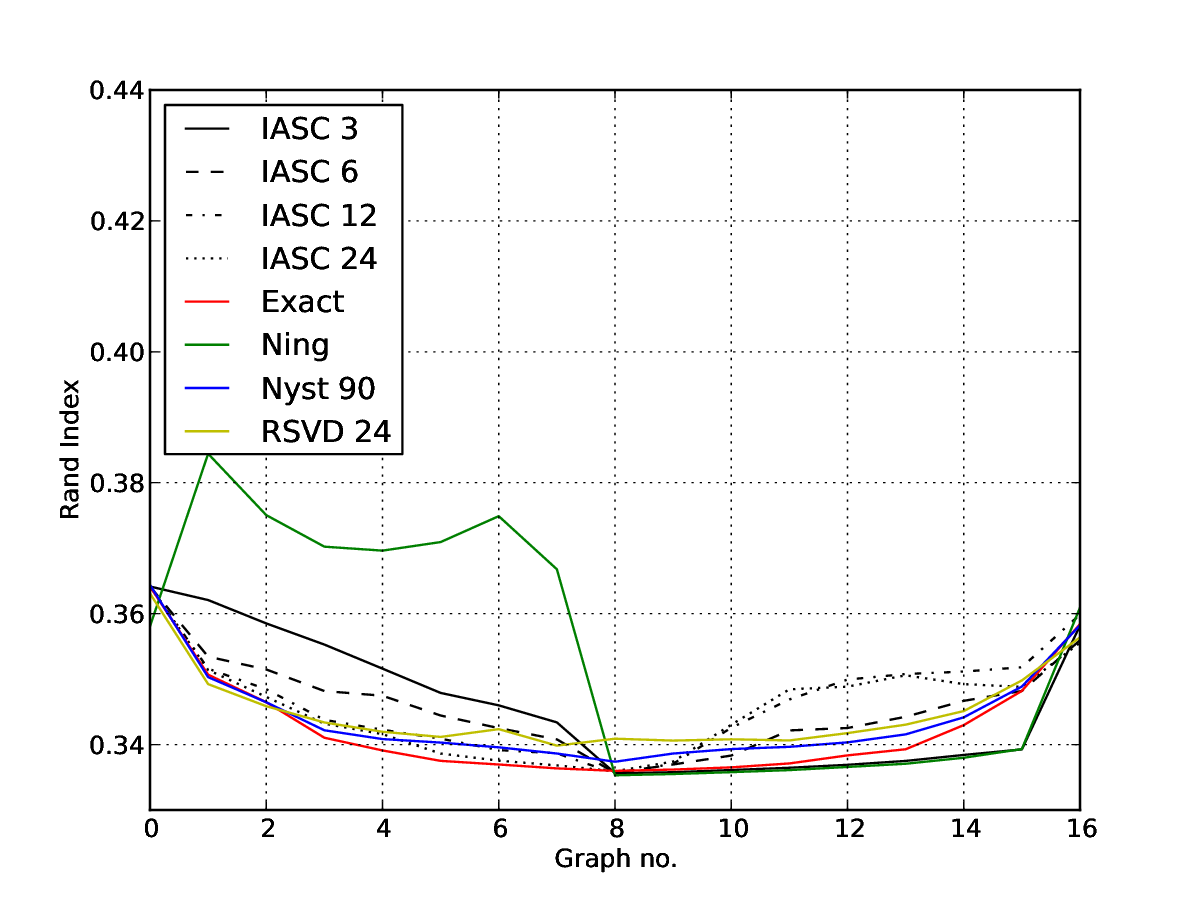}}%
\hfill%
\subfigure[Rand Index, $p=0.2$]{\includegraphics[width=\subfigsize\linewidth]{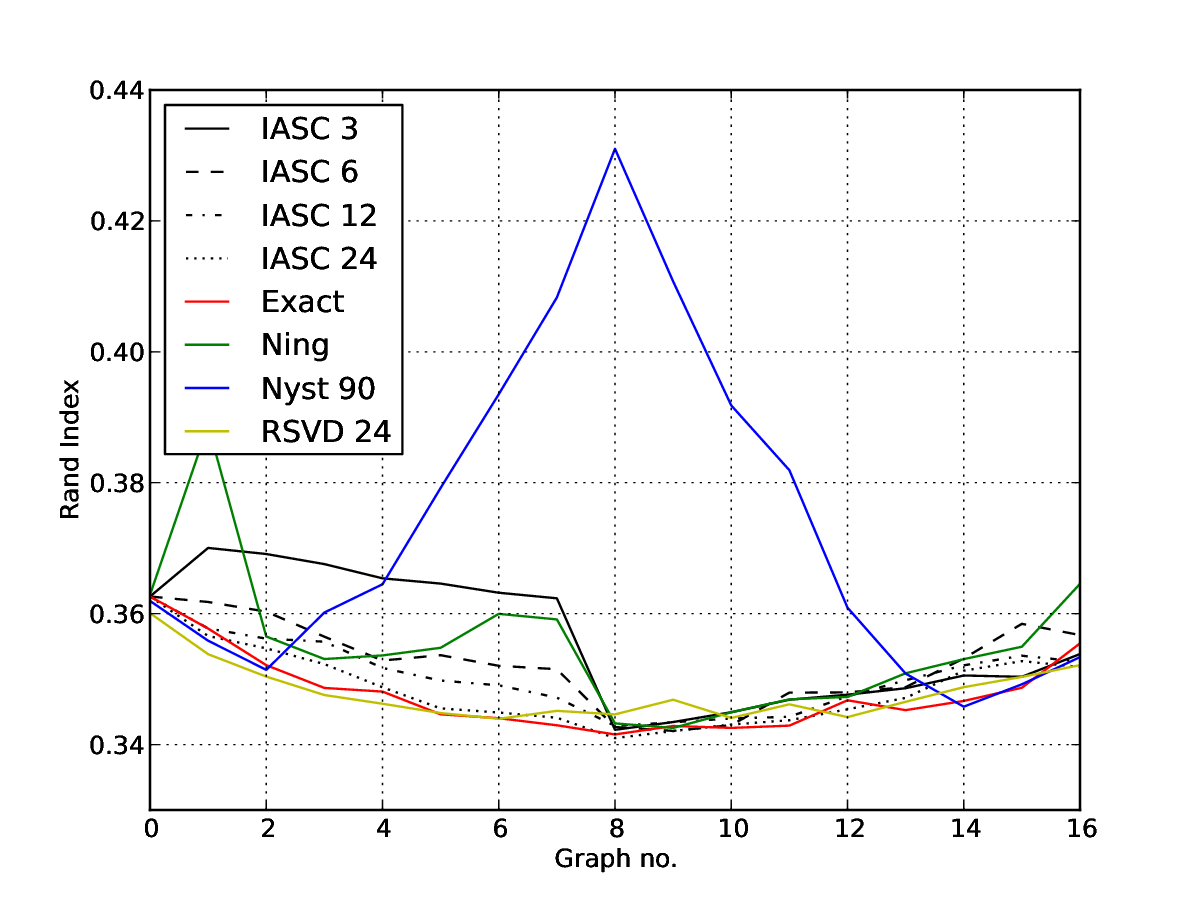}}%
\\
\subfigure[Canonical angles, $p=0.1$]{\includegraphics[width=\subfigsize\linewidth]{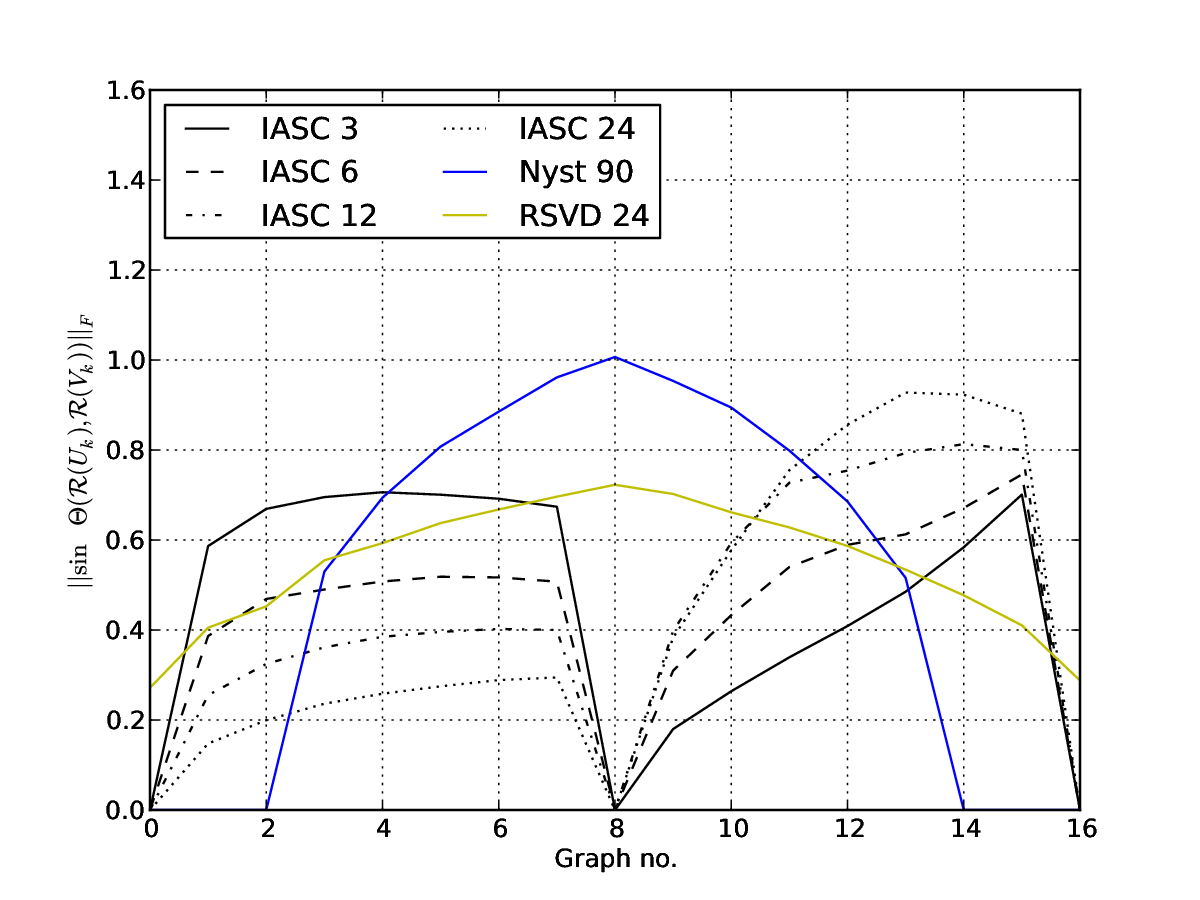}}%
\hfill%
\subfigure[Canonical angles, $p=0.2$]{\includegraphics[width=\subfigsize\linewidth]{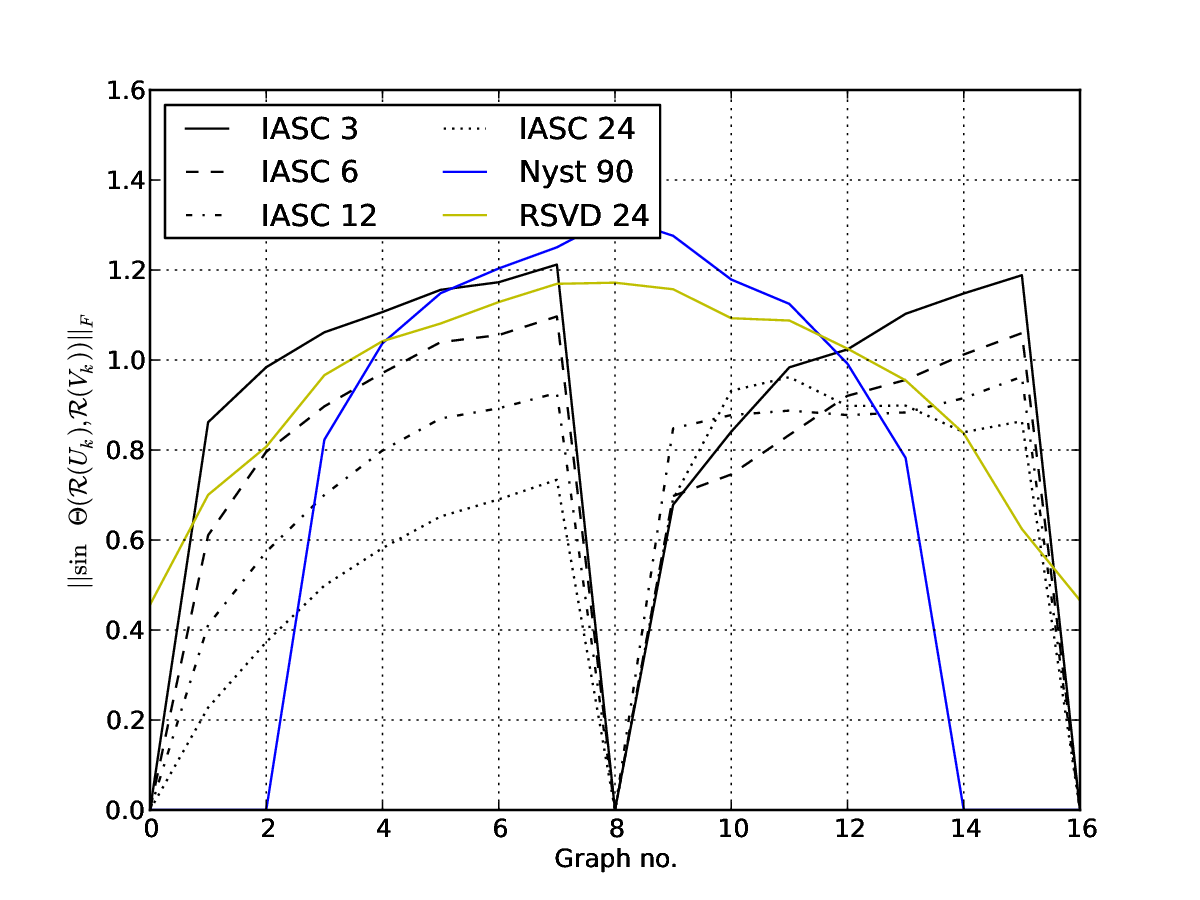}}%
\end{center} 
\caption{The mean Rand Index error of the learned clustering on \texttt{3clust}. The numbers after \texttt{IASC} denote the number of eigenvectors computed. }
\label{fig:syntheticErrors}
\end{figure}

The errors on \texttt{3clust} are shown in Figure \ref{fig:syntheticErrors}. The first point to note is that the errors decrease as the cluster size increases up until graph 8 and then increase again as vertices are removed. This is explained by the fact that as the cluster size increases there are more edges within each cluster relative to those between clusters, hence they becomes easier to identify. Notice also that the approximation methods perform worse (in terms of Rand Index) than \texttt{exact} except for \texttt{Ning} and \texttt{IASC} $\ell=3$ when vertices are removed. This exception is due to a coincidence: for any iteration $i \in 8,...,16$, clustering results are better when using the eigen-decomposition at time 8 as they are based on more data. On the other hand, \texttt{Ning} and \texttt{IASC}, $\ell=3$, do not accurately update the eigen-decomposition and hence the eigenvectors are close to those at graph 8. As a consequence, both approaches lead to better results than \texttt{exact}. Apart from this exception, \texttt{IASC}, $\ell = 24$, and \texttt{RSVD} have results close to \texttt{exact}, while \texttt{Nystr\"{o}m} leads to a bad Rand Index score for $p=0.2$. \texttt{Ning} has non-smooth behaviour for any value of $p$. Lastly, the results of \texttt{Nystr\"{o}m} when $p=0.1$ demonstrate the limit of canonical angles to measure the clustering quality of an approximation approach: having large canonical angles values does not necessarily imply a Rand Index score far from \texttt{exact}. 

It is worth noting that since $\ell$ is a fixed value for \texttt{IASC}, the approximation of the largest eigenvectors of the shifted Laplacian represents a smaller fraction of the total sum of eigenvalues as the cluster size increases. When using 24 eigenvectors for the approximation on the 2nd largest graph one requires approximately $14.5\%$ of the dimensionality of the Laplacian, yet Rand Index is close to that of \texttt{exact}. 

\begin{figure}
\begin{center}
\subfigure[Rand Index]{\includegraphics[width=\subfigsize\linewidth]{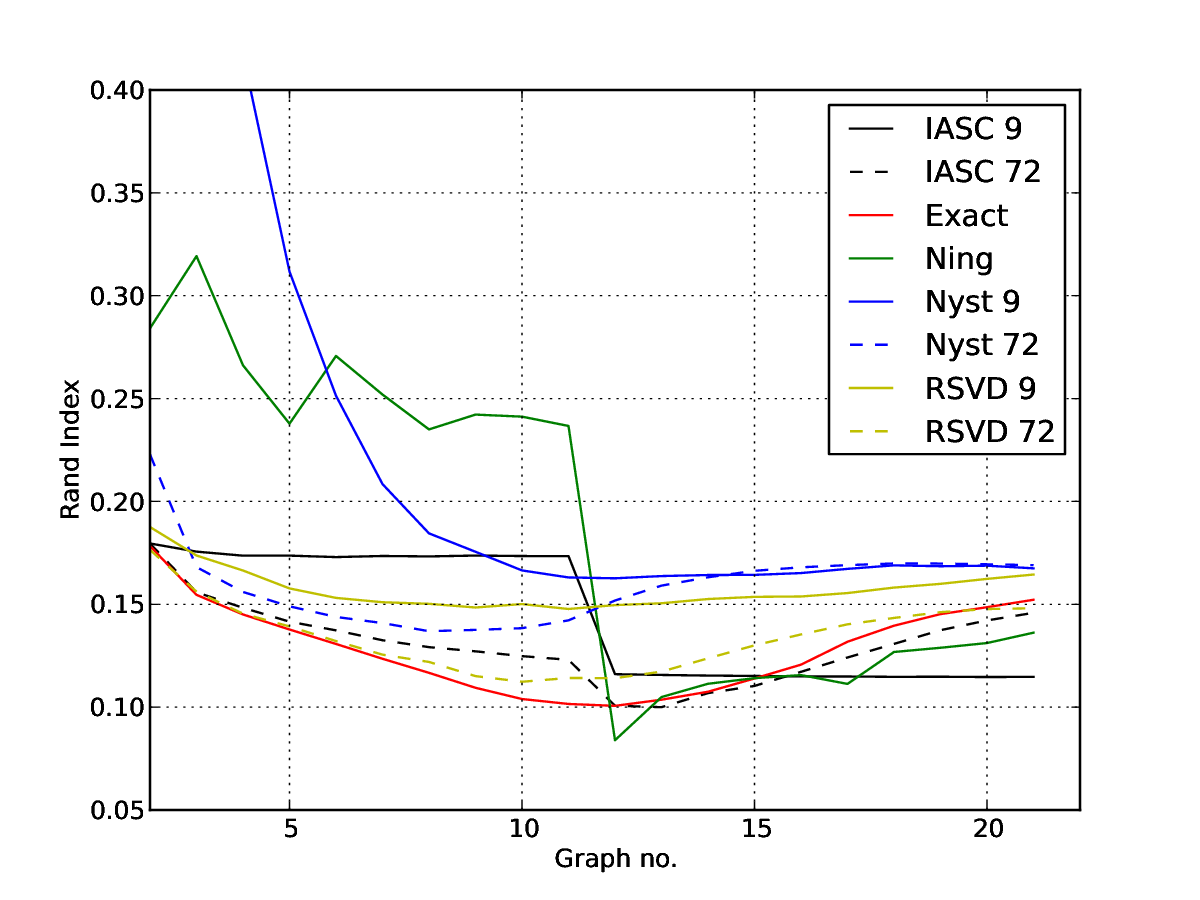}}%
\hfill%
\subfigure[Canonical angles]{\includegraphics[width=\subfigsize\linewidth]{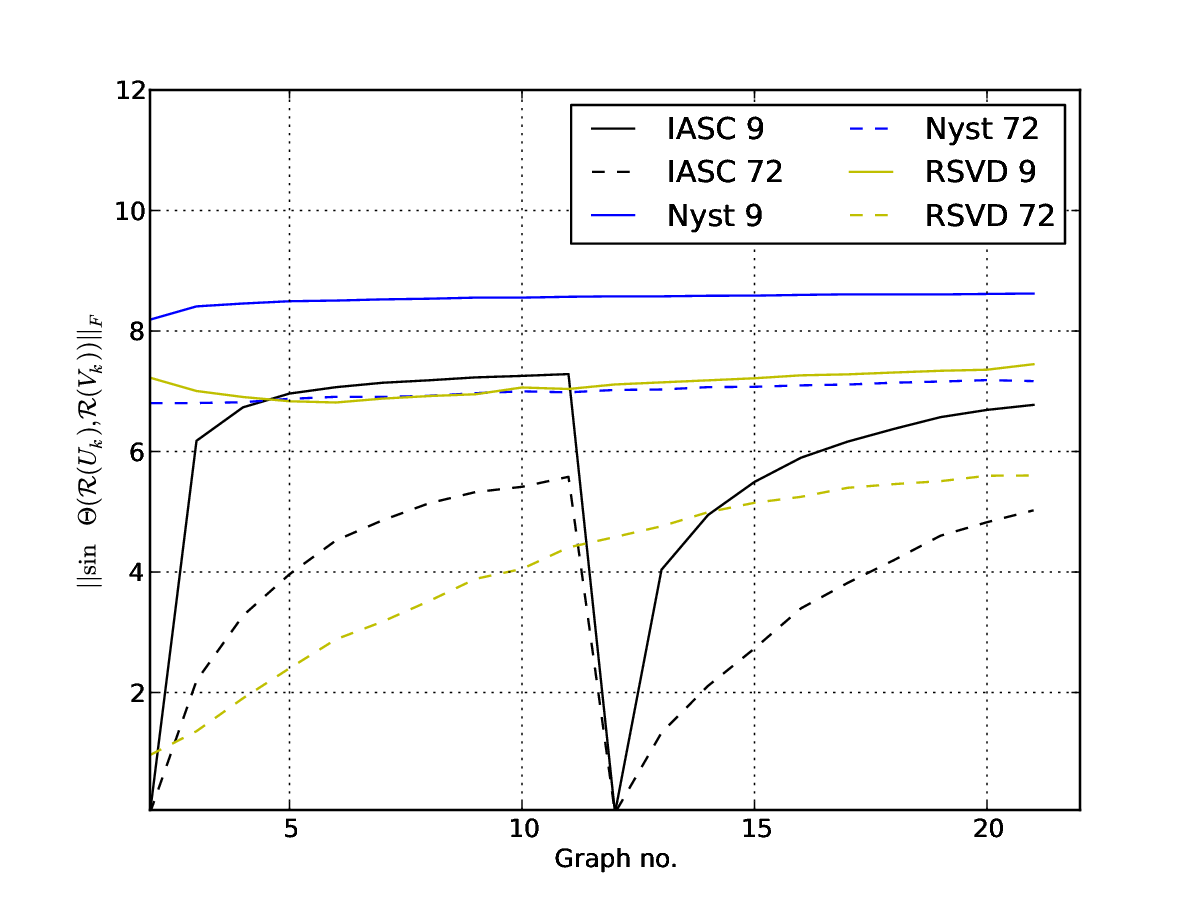}}%
\end{center} 
\caption{The mean Rand Index of the learned clusterings on \texttt{discrepancy}. For \texttt{IASC} and \texttt{Nystr\"{o}m} the numbers in the legend indicates the number of eigenvectors and columns sampled at each iteration. }
\label{fig:discrepancyErrors}
\end{figure}

The results on \texttt{discrepancy} (Figure \ref{fig:discrepancyErrors}) show that \texttt{Ning} generally does not generate accurate clustering until the 12th graph when eigenvectors are recomputed, however it becomes competitive after this point. This is partly due to the fact that one solves a different eigen-system to the other methods which is not as effective at clustering for the initial graphs. As we increase the number of columns used for the Nystr\"{o}m approach or the random projections used by \texttt{RSVD}, the Rand Index values improve for the first few graphs in which the clusters are generally not well defined. In contrast, \texttt{IASC} results are relatively accurate with 72 eigenvectors. Surprisingly at the final graph we see that with just 9 eigenvectors \texttt{IASC} results in the most accurate clustering. At this point more eigenvectors seem to make the solution worse, with the exact eigen-decomposition being less accurate than \texttt{IASC} with 72 eigenvectors. One explanation is that an accurate eigen-decomposition fits ``noise'' in the Laplacian which is excluded with \texttt{IASC} and a low value of $\ell$.   

\subsection{Real-world Graphs} 
We now apply the clustering methodology analysed previously on several real datasets.
\subsubsection{Setup} 

Here we use the clustering methods on three real datasets. The first one represents individuals in Cuba who are detected as HIV positive between the period 1986 and 2004, see \cite{Auvert07} for details of the related database. Edges in the graph indicate the occurrence of a sexual contact between two individuals as determined using \emph{contact tracing}, whereby contacts of an infected person are identified and tested. The full sexual contact graph at the end of 2004 consists of 5389 people however it is strongly disconnected and we consider the growth of the largest component of size 2387. We find graphs of detected individuals at 1 month intervals, starting when the graph contains at least 500 vertices. At any particular time point, we consider the component containing first person detected in the largest component at the end of the recorded epidemic. 

The next dataset is the high energy physics theory citation network from the Arxiv publications database \cite{gehrke2003overview} for the period February 1992 to March 2002. The full dataset contains 27,770 papers with 352,807 edges and a graph is generated as follows: if a paper cites another, an edge is made from the former to the latter. Not all of the papers present in the dataset have publication dates and hence we use only those with dates and label cited papers with the date of the oldest citing paper. Taking the largest component, the final resulting graph consists of 15,112 vertices and 193,826 edges. To track the evolution of the connected component we start with the oldest paper, and consider the graphs of connected papers at 1 month intervals, starting with a graph of at least 500 vertices. 

The final dataset, called \texttt{Bemol}, was first introduced in \cite{richard10} and corresponds to the purchase history of an e-commerce website over a period of almost two years. The initial dataset is a bipartite graph between users and products composed of more than 700,000 users and 1,200,000 products. In the current experiment we focus on the first 10,000 users, and a graph is constructed between users with edge weights corresponding to the number of commonly purchased products between two users. Taking a maximum of 500 purchases per iteration in the graph sequence we focus on graphs 500 to 600. 

To test the clustering methods we run each on the sequences of evolving graphs under a variety of parameters. As the selection of the number of clusters is a complex issue and outside the scope of this paper and we manually choose this value for each dataset. The experiment is run using $k = 25$ clusters for  \texttt{HIV}, $k=50$ for \texttt{Citation} and $k=100$ for \texttt{Bemol}. For \texttt{HIV} and \texttt{IASC}, $\ell \in \{25, 50, 100\}$ and $R = 10$ and for \texttt{Ning} we recompute exact eigenvectors after every $10$ iterations. When using \texttt{Citation} and \texttt{Bemol}, eigenvector recomputations are performed every 20 iterations and  $\ell \in \{100, 200, 500\}$. With \texttt{Nystr\"{o}m} the number of columns is chosen from $m \in \{1000, 1500\}$ for \texttt{HIV} and $m \in \{2000, 5000\}$ for the other datasets. Finally, we apply the randomised SVD method with $q=2$ and $r \in \{1000, 1500\}$ for \texttt{HIV} and $r \in \{2000, 5000\}$ for the remaining datasets. Note that in our implementation of $k$-means clustering, $k$ represents an upper bound on the number of clusters found. To evaluate the learned clusters we use measures of \emph{modularity} and \emph{$k$-way normalised cut}. Let $\dv_i = \sum_j \Wm_{ij}$ and $r = \sum_i \dv_i$, then the modularity is defined as 
\begin{displaymath} 
Q = \frac{1}{2r} \sum_{i,j} \left(\Wm_{ij} - \frac{\dv_i \dv_j}{2r}\right) \delta(\cv_i, \cv_j),
\end{displaymath} 
where $\cv \in \mathbb{R}^{n}$ is the cluster indicator vector and $\delta$ is the Kronecker delta function. Intuitively modularity is the difference in the sum of edges weights within a cluster and the expected edge weights assuming the same weight distribution $\dv$ for each vertex. The $k$-way normalised cut is
\begin{displaymath}
N = \frac{1}{k} \sum_{\ell=1}^k \frac{\sum_{ij} \Wm_{ij} \delta(\cv_i, \cv_\ell) (1 - \delta(\cv_j, \cv_\ell))}{\sum_{ij} \Wm_{ij}\delta(\cv_i, \cv_\ell)}. 
\end{displaymath} 
A cut between two clusters $A$ and $B$ is the sum of the weights between the clusters and the normalised cut is this sum divided by the sum of the weights of all edges incident to vertices in cluster $A$. Hence the $k$-way normalised cut is the mean normalised cut between each cluster and its complementary vertices. To summarise, the greater the modularity, the better, and the lower the  $k$-way normalised cut the better. 

All experimental code is written in Python and we use an Intel Core i7-2600K at 3.40GHz with 16GB of RAM to conduct the simulations. The Laplacian matrices are stored in compressed sparse row representation and eigenvectors are found using Implicitly Restarted Lanczos Method in ARPACK \cite{lehoucq1998arpack} which computes only the required eigenvectors and not the full eigen-decomposition. 

\subsubsection{Results} 

\begin{figure}
\begin{center}
\subfigure[\texttt{HIV} modularity]{\includegraphics[width=\subfigsize\linewidth]{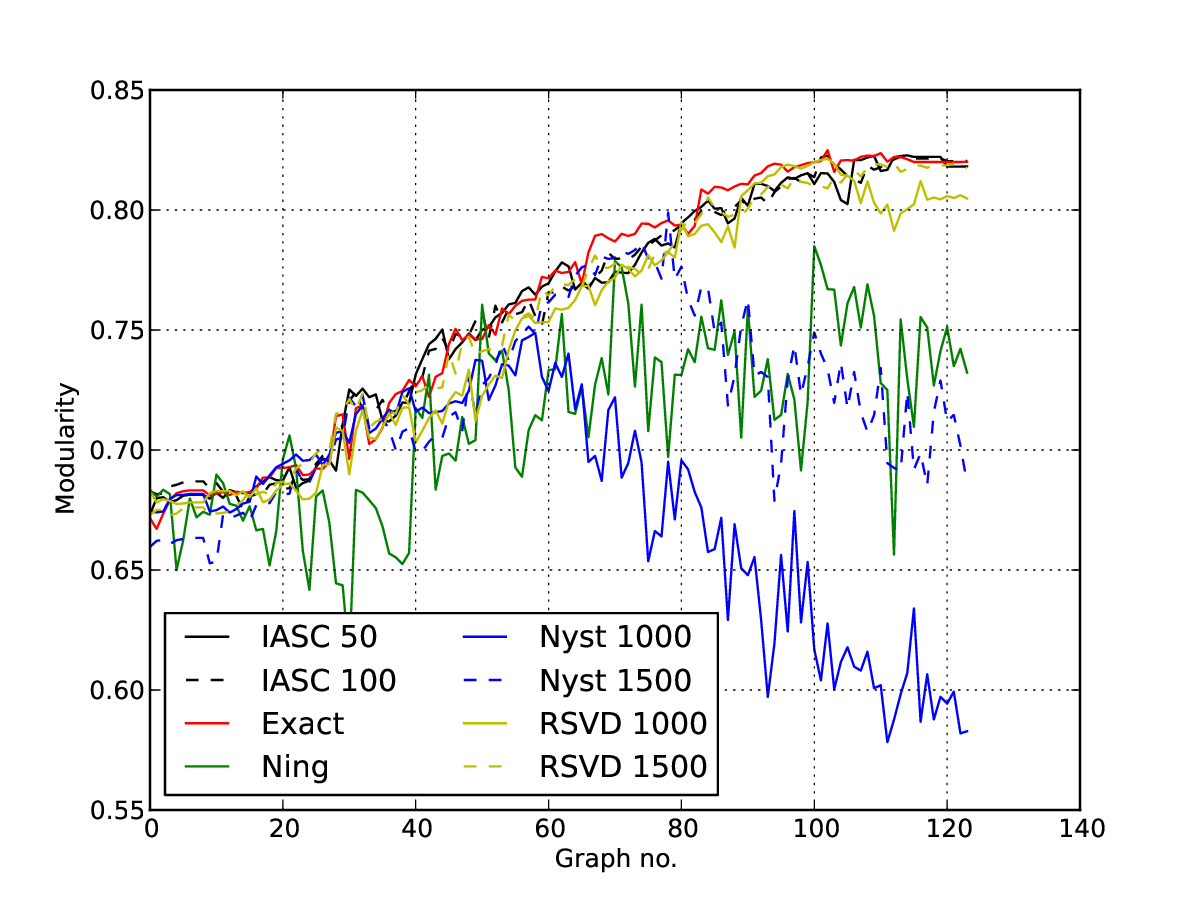}}%
\subfigure[\texttt{HIV} $k$-way normalised cut]{\includegraphics[width=\subfigsize\linewidth]{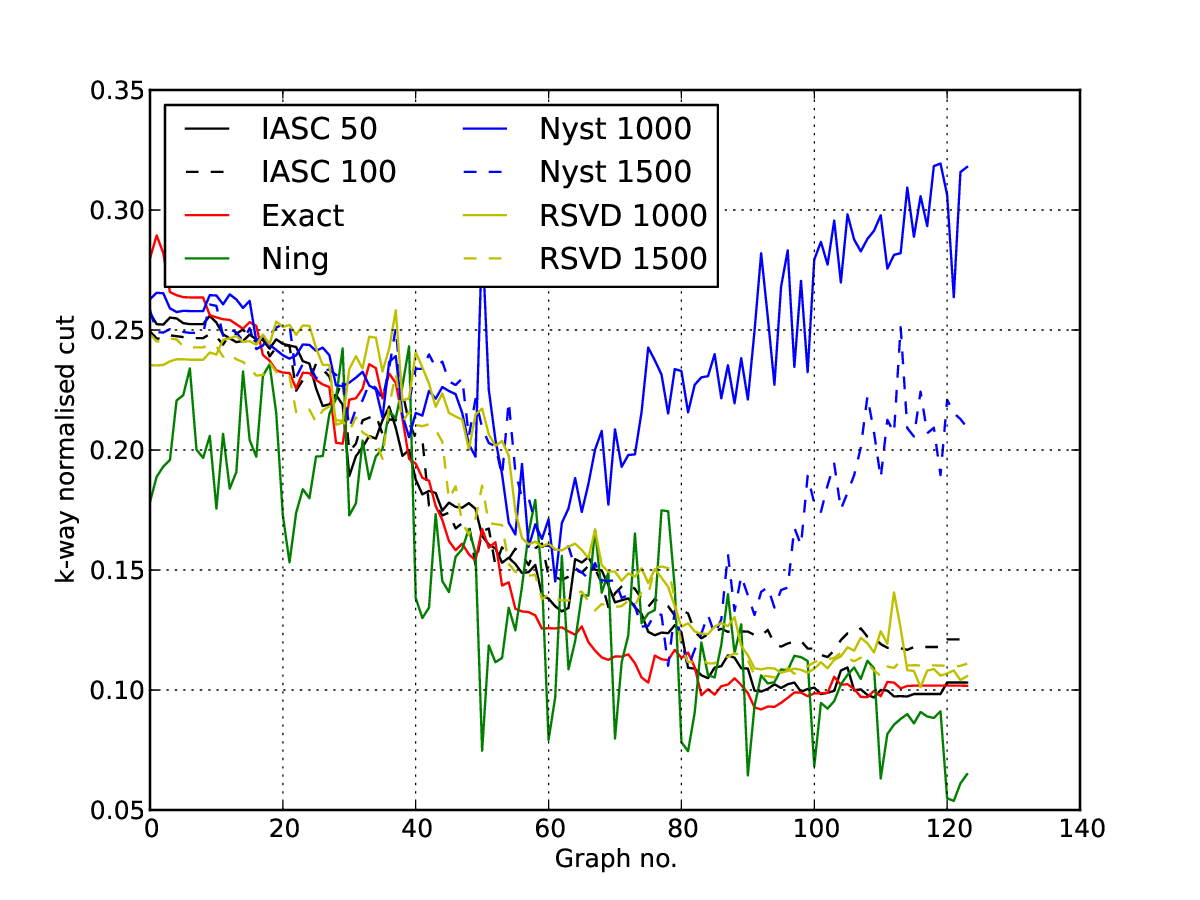}}
\subfigure[\texttt{Citation} modularity]{\includegraphics[width=\subfigsize\linewidth]{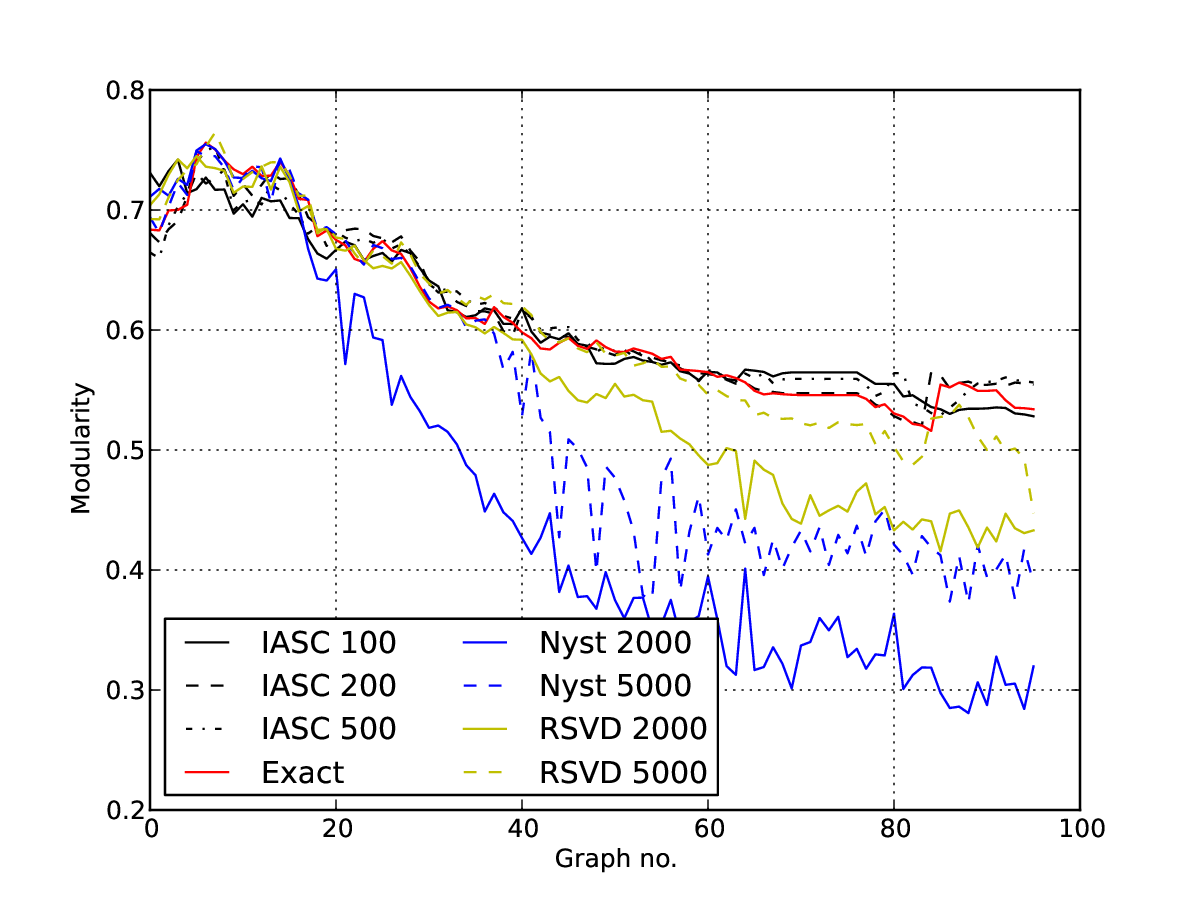}}%
\subfigure[\texttt{Citation} $k$-way normalised cut]{\includegraphics[width=\subfigsize\linewidth]{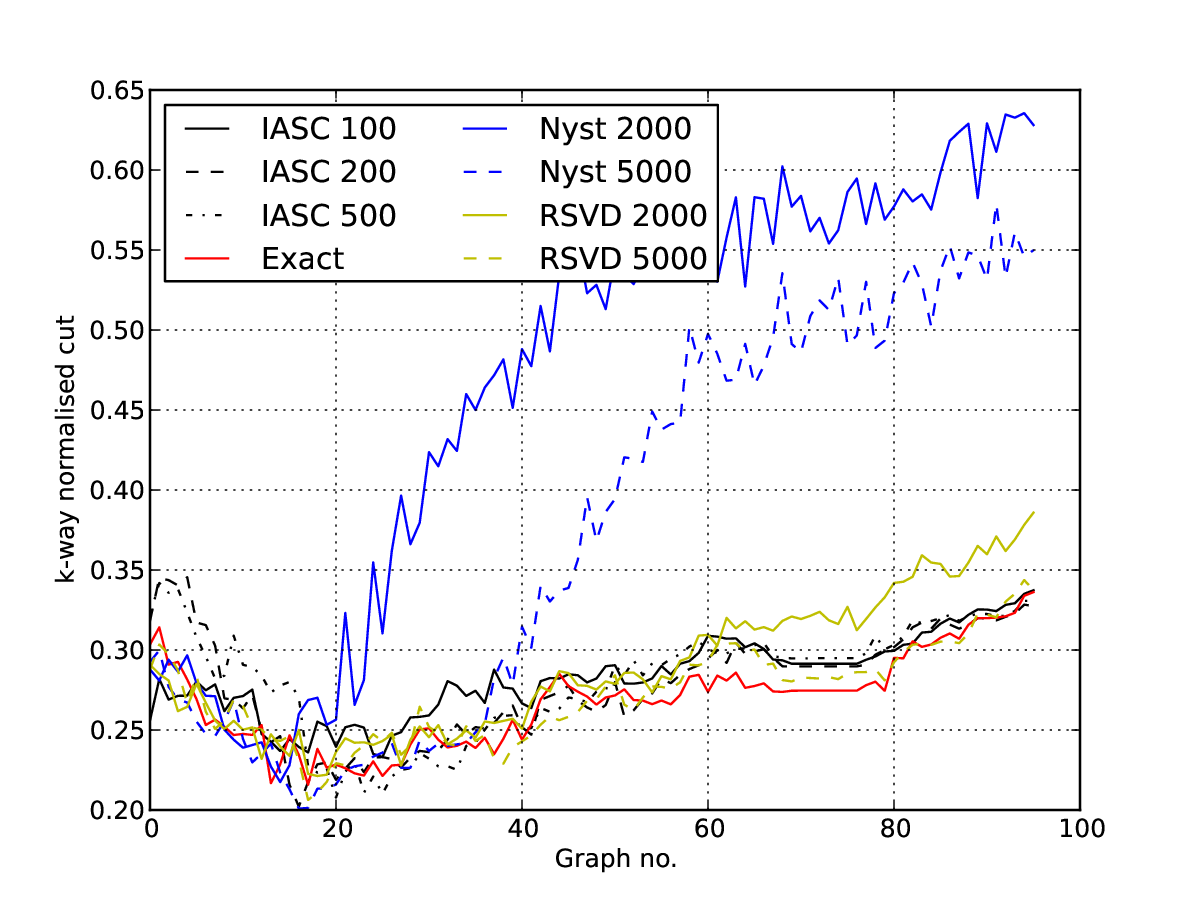}}
\subfigure[\texttt{Bemol} modularity]{\includegraphics[width=\subfigsize\linewidth]{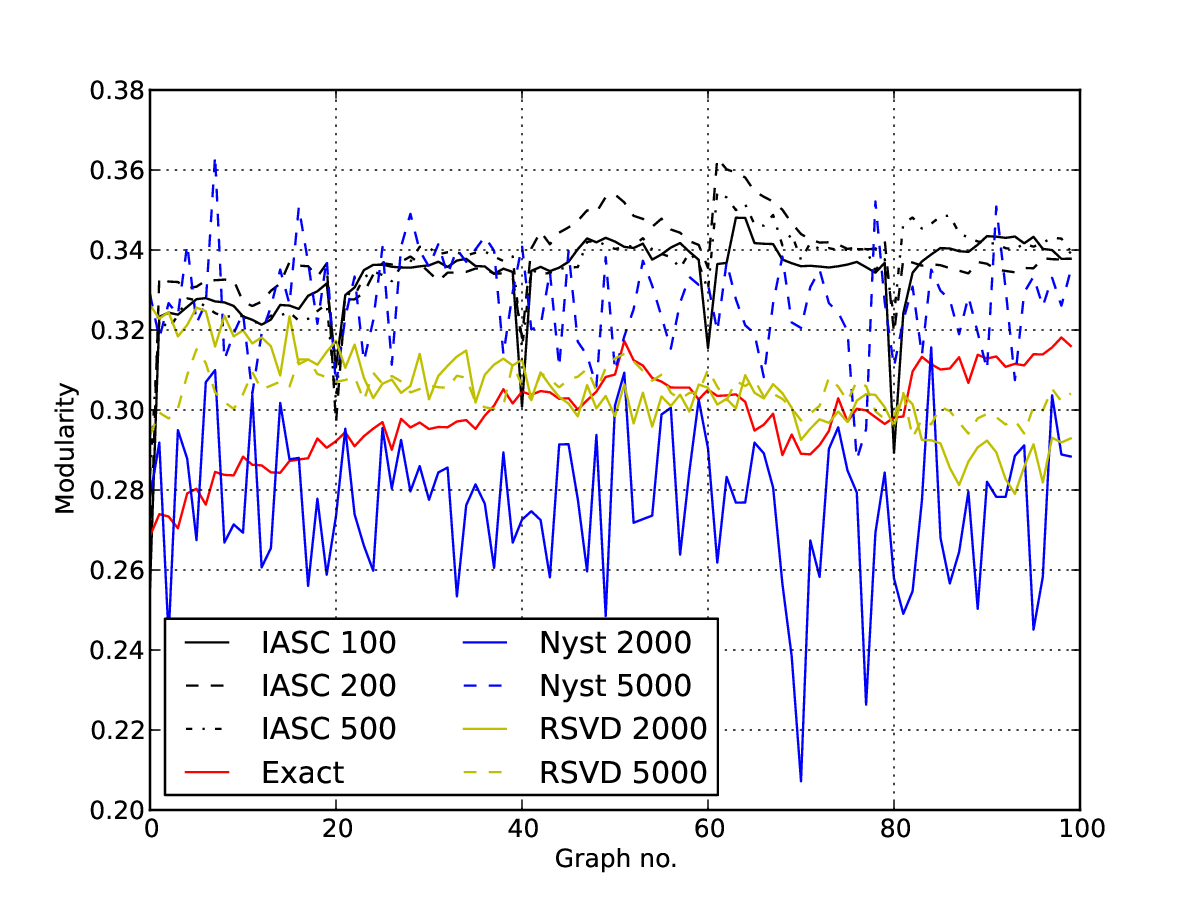}}%
\subfigure[\texttt{Bemol} $k$-way normalised cut]{\includegraphics[width=\subfigsize\linewidth]{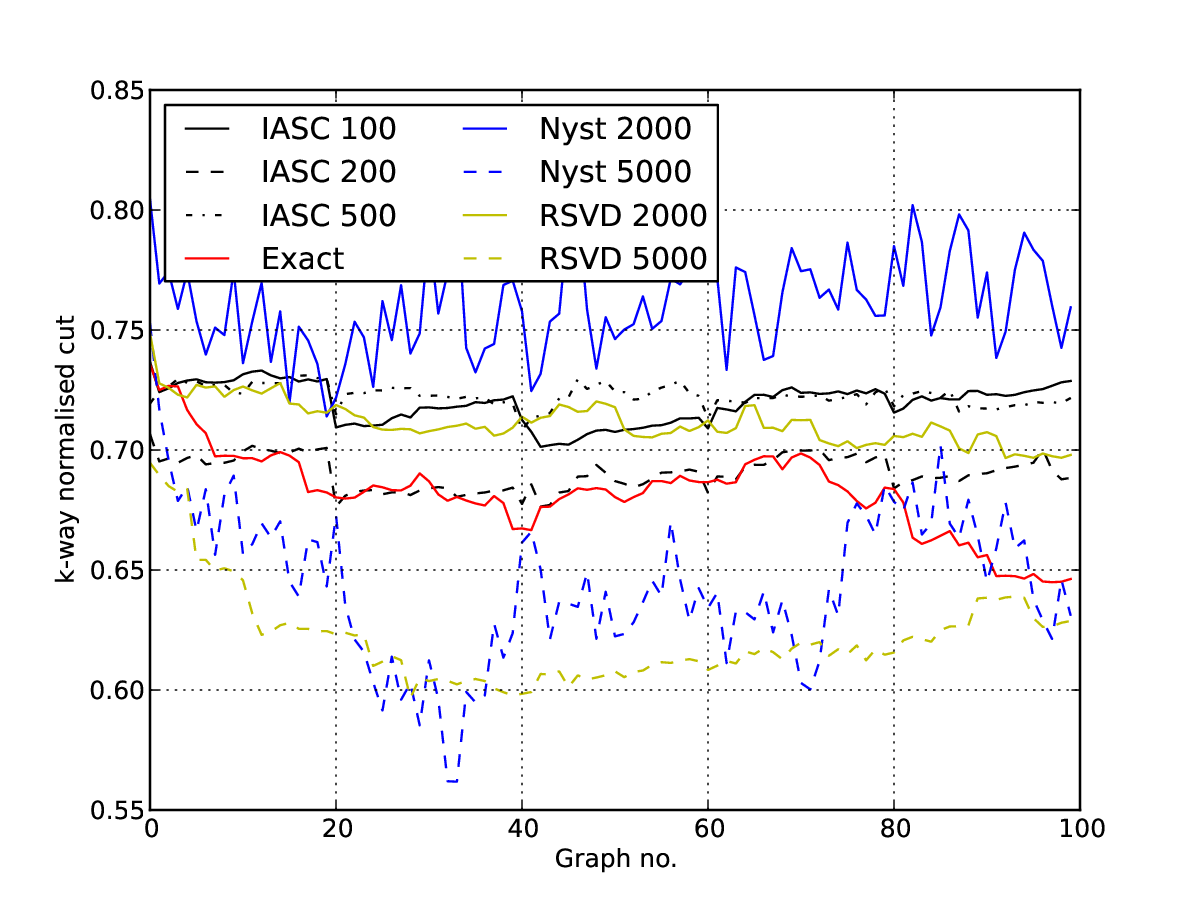}}
\end{center} 
\caption{The performance of the clustering methods on the sequences of changing graphs.}
\label{fig:RealDataResults}
\end{figure}

Figure \ref{fig:RealDataResults} shows the resulting modularities and $k$-way normalised cuts for all datasets however we begin by studying $\texttt{HIV}$. For both \texttt{IASC} and \texttt{Ning}, since eigenvectors are recomputed every 10 iterations, this can manifest itself as sudden changes in the modularities and $k$-way normalised cuts. These changes are more pronounced with \texttt{Ning}.  We see a close correspondence of \texttt{IASC} and \texttt{exact} for both the modularity and normalised cut. As we observed with the toy datasets, a lower value of $\ell$ seems to improve results with the final graph having a cut of $0.09$ with \texttt{IASC} $\ell=100$ versus $0.10$ for \texttt{exact}. Notice that \texttt{IASC} matches or improves results on \texttt{exact}, while keeping only 5\% or fewer of the final number of eigenvectors. \texttt{Ning} fares badly in terms of the modularity of the resulting clustering with a value of 0.73 versus 0.82 for the \texttt{exact} approach at the final graph. However, with the cut measure \texttt{Ning} provides the best clustering, albeit with a more unstable curve than the other methods. \texttt{Nystr\"{o}m} does not provide a good approximation of the largest eigenvectors when the rank of the Laplacian exceeds the number of columns sampled. In contrast, \texttt{RSVD} is broadly competitive with \texttt{IASC} when using 1500 random projections. One of the reasons for the effectiveness of \texttt{RSVD} is Step \ref{step:power} of Algorithm \ref{alg:randomSVD} which helps to ensure the column space of $\Ym$ is close to that generated using the largest eigenvectors of $\Am$, see \cite{halko2011finding} for further details. 

A similar picture emerges with \texttt{Citation} and we see again that \texttt{IASC} is close to \texttt{exact} in terms of both measures. Note that it was too costly to compute clustering using \texttt{Ning} on this dataset and \texttt{Bemol}. We ran \texttt{Ning} on \texttt{Citation} for 337,920s before terminating the experiment: a computational time of at least 7.5 times more than the next most costly approach of \texttt{RSVD}, $r = 5000$, which took 44,880s. On \texttt{Citation}, when $\ell\in\{200, 500\}$ we obtain a close match to \texttt{exact} in general. The results are impressive when we consider the change in the graphs between eigenvector recomputations: the first graph is of size 555, and the 19th is 2855, an increase of 2300. With the 60th graph there are 10,063 vertices and 12,135 at the 79th. Looking at the \texttt{Nystr\"{o}m} curves, we again observe poor clustering performance even when $m=5000$. Furthermore, \texttt{RSVD} $r = 5000$ can compete well with the \texttt{exact} method particularly when considering the cut measure although the modularity using \texttt{RSVD} suffers after approximately the 60th graph. 

Finally consider the \texttt{Bemol} graphs in which it is difficult to find clear clusters, although they become more distinct over time. This is evident when looking at the \texttt{exact} curves for example: the modularity increases slightly from 0.27 to 0.32 whereas the $k$-way normalised cut falls from 0.74 to 0.65 from beginning to end. In contrast to the other datasets \texttt{exact} is improved upon by both \texttt{IASC} and \texttt{Nystr\"{o}m} (when $m = 5000$) respectively. One of the reasons that the \texttt{Nystr\"{o}m} and \texttt{RSVD} methods are effective on this data is because there are many edges and one can sample them out without affecting the clustering significantly. Note however that as we have seen in the other plots, \texttt{Nystr\"{o}m} is rather unstable compared to the other methods. Furthermore, the eigenvector updates every 20 iterations are noticeable in the cluster measures with \texttt{IASC}.

\begin{figure}[ht]
\begin{center}
\subfigure[\texttt{HIV}]{\includegraphics[width=\subfigsize\linewidth]{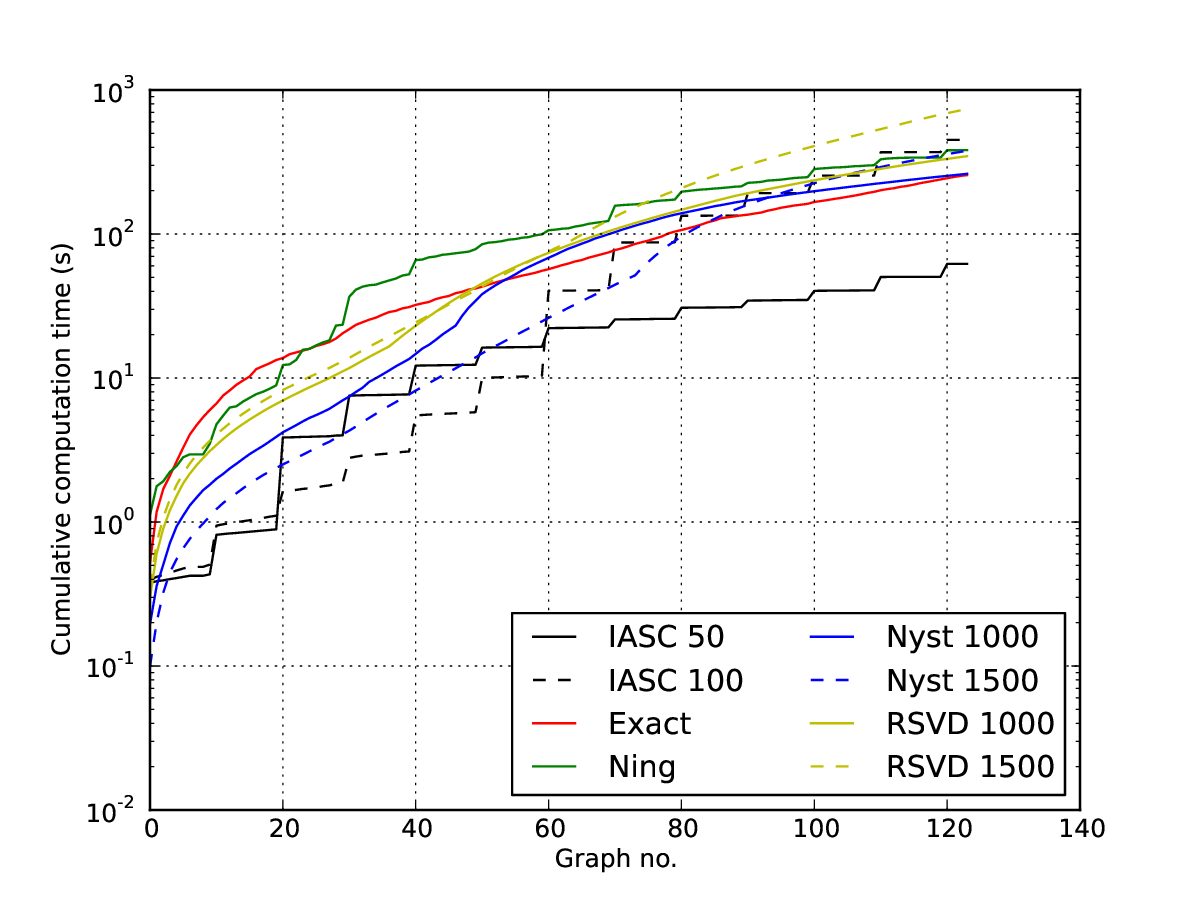}}
\subfigure[\texttt{Citation}]{\includegraphics[width=\subfigsize\linewidth]{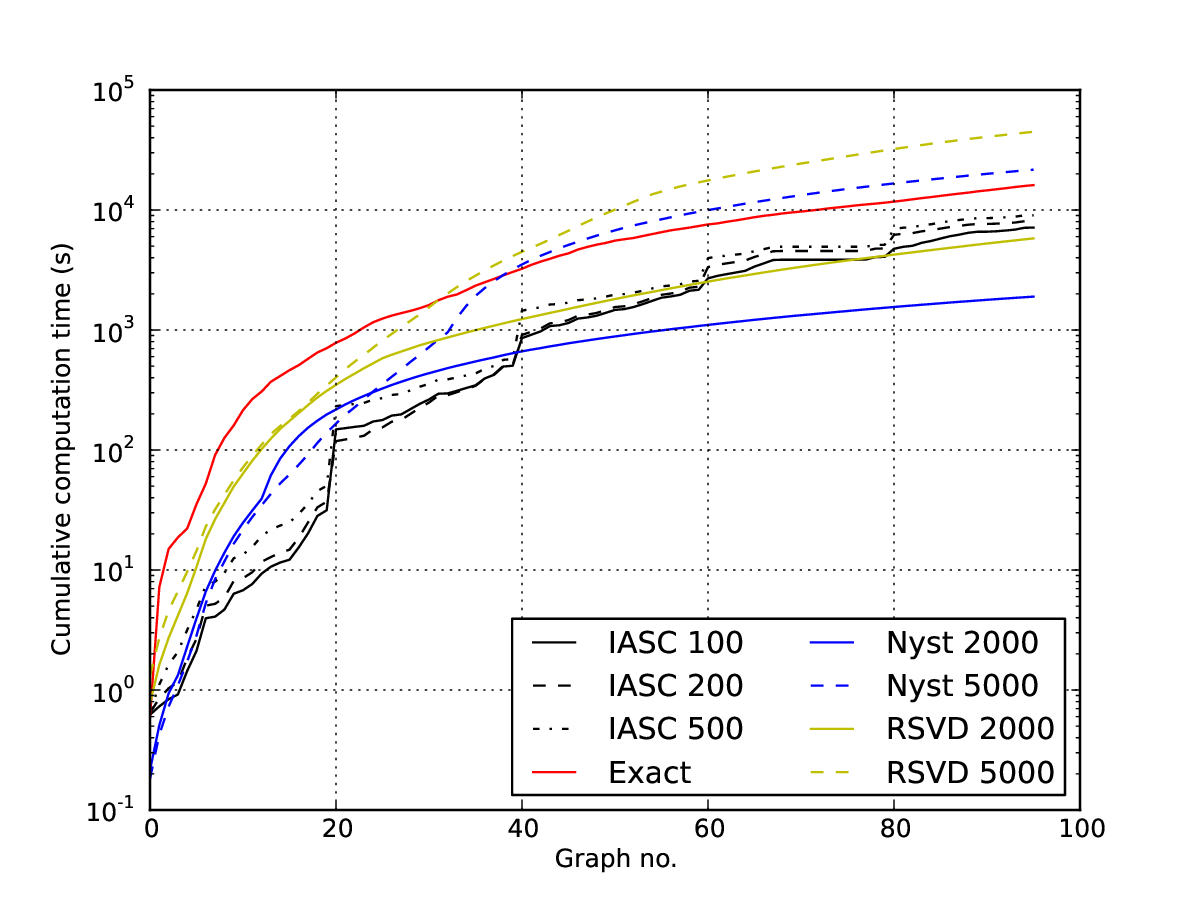}}%
\subfigure[\texttt{Bemol}]{\includegraphics[width=\subfigsize\linewidth]{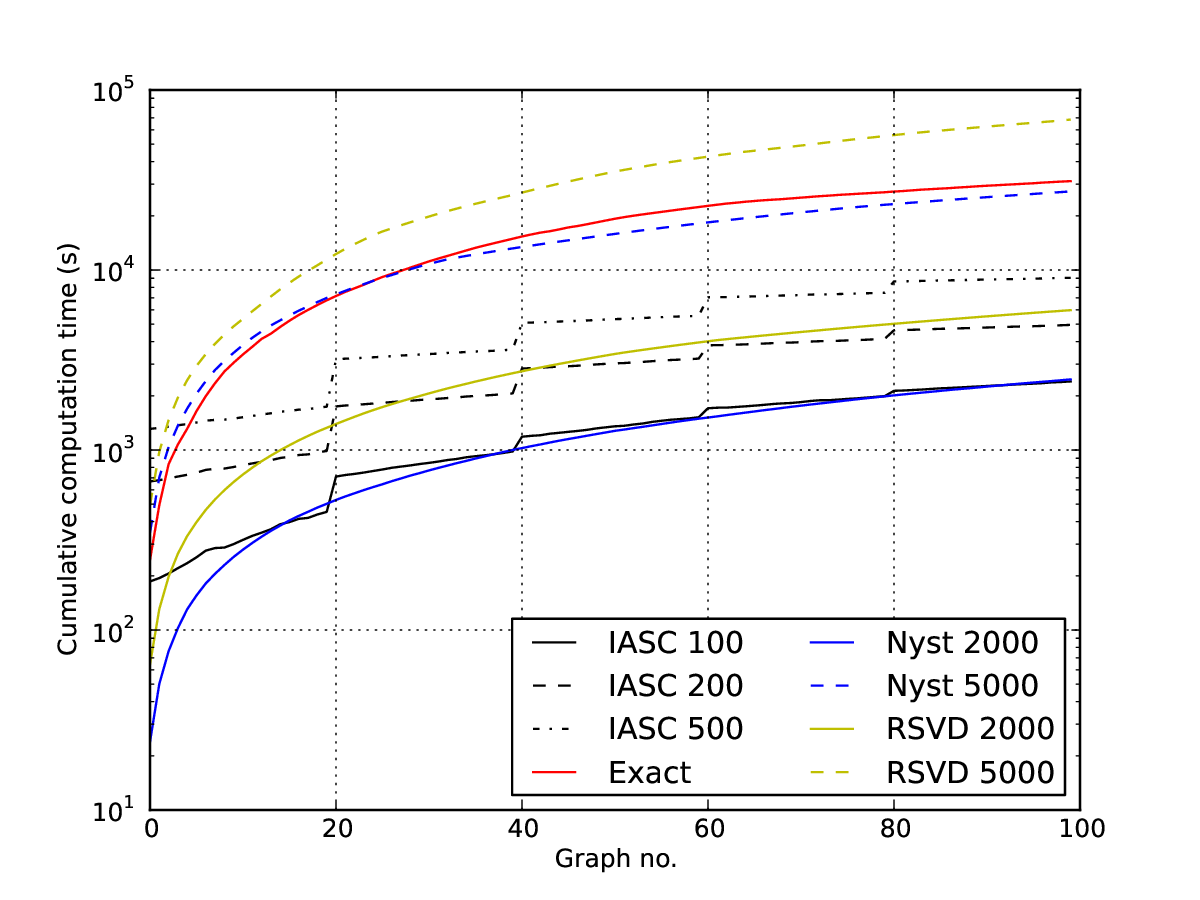}}
\end{center} 
\caption{The cumulative times taken by the eigenvector computations of the clustering methods.}
\label{fig:RealDataTimes}
\end{figure}

To conclude the analysis, Figure \ref{fig:RealDataTimes} shows the timings of the eigenvector computations of the clustering methods for the datasets. With \texttt{HIV} we can make a comparison with \texttt{Ning} and one can see that a cost of the method is a computation time which exceeds that of \texttt{exact}. \texttt{IASC} has a cumulative computation time of 62.0s, \texttt{exact} took 256.4s  compared to 382.1s for \texttt{Ning}. With the \texttt{Bemol} and \texttt{Citation} datasets the Nystr\"{o}m approach costs the least in terms of computation when $m = 2000$ however exceeds the time taken for \texttt{exact} when $m=5000$. In this case the time is dominated by the eigen-decomposition of a matrix in $\mathbb{R}^{m \times m}$. Of note also is that \texttt{RSVD} with $r=5000$ exceeds the time required for \texttt{exact} yet this was the number of random projections required for competitive performance. \texttt{IASC} improves over \texttt{exact} over the whole sequence of graphs as one does not recompute the eigenvectors at every iteration. Notice that the ``staircase'' effect in the \texttt{IASC} curves correspond the computation of the exact eigenvectors. Observe that on \texttt{Bemol}, \texttt{IASC} $\ell = 200$ takes 4,956 seconds in total for eigenvector computations versus 31,182 for \texttt{exact}, a speedup factor of 6.29 for a similar cluster quality. The equivalent improvement is 2.26 on \texttt{Citation}.

To emphasise the conditions in which \texttt{IASC} can be effective, we again cluster over the \texttt{Citation} data however graphs are recorded at 5 day intervals. The parameters are identical to those used above except we set $R= 50$. Figure \ref{fig:citation5day} shows the resulting clustering qualities and eigenvector computation timings for \texttt{IASC} and \texttt{RSVD}. As one might expect we observe that \texttt{RSVD} is competitive to \texttt{IASC}, $\ell = 100$, when $r = 5000$, and only competitive with $r=2000$ until approximately the 400th graph. The timings however show that \texttt{IASC} is faster than \texttt{RSVD} for all values of $\ell$ over all graphs. Furthermore, when $\ell = 100$ \texttt{IASC} took 9640 seconds compared to 32,006 and 158,510 seconds with \texttt{RSVD} using $r$ values for $2000$ and $5000$ respectively. 

\begin{figure}
\begin{center}
\subfigure[\texttt{$k$-way normalised cut}]{\includegraphics[width=0.45\linewidth]{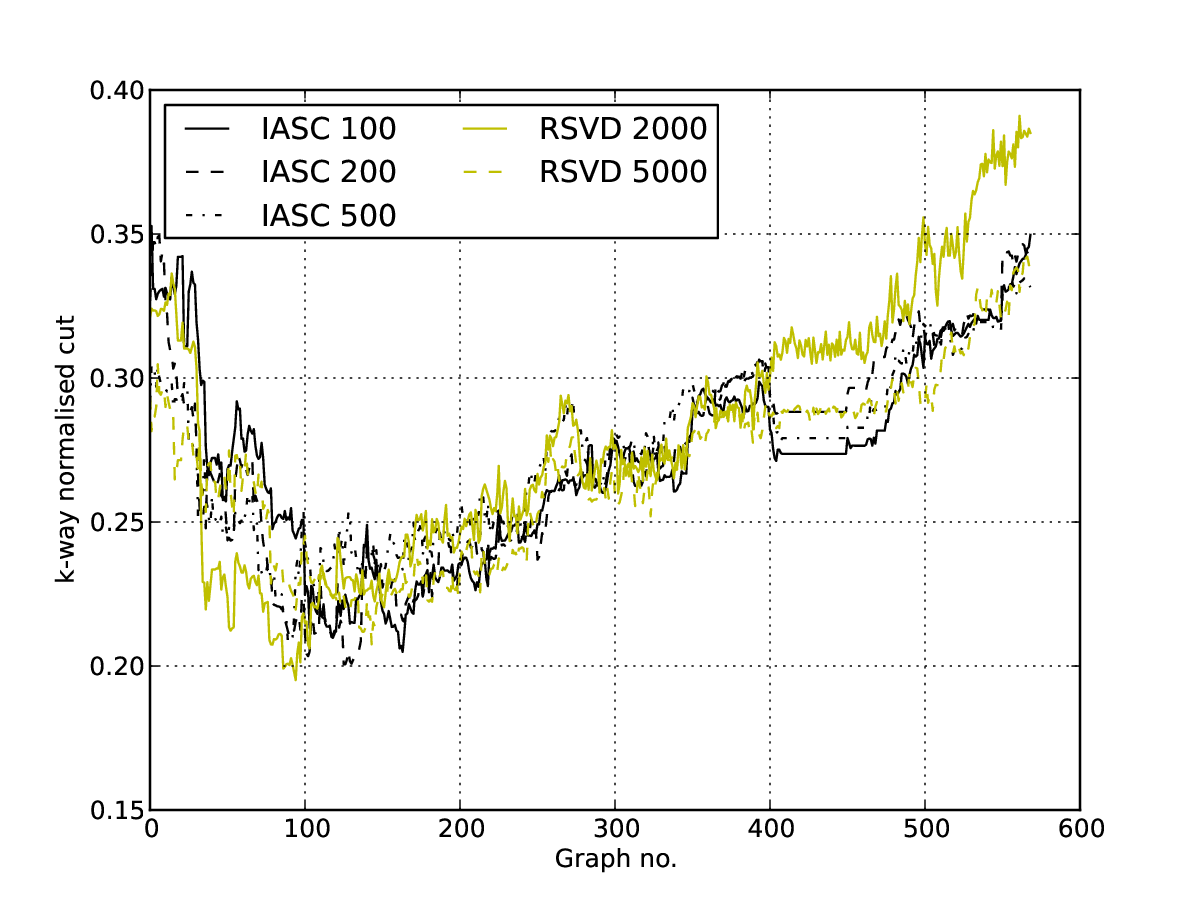}}
\subfigure[\texttt{Cumulative time}]{\includegraphics[width=0.45\linewidth]{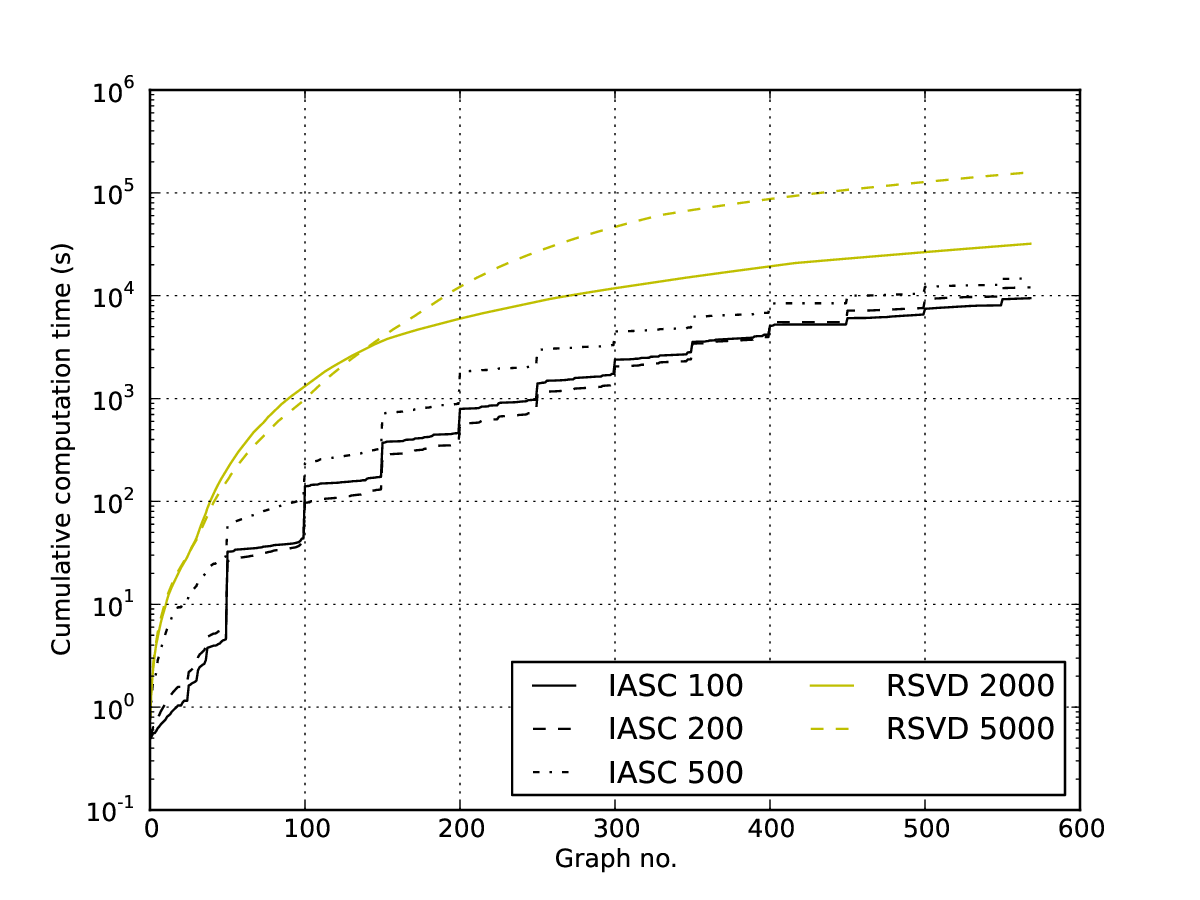}}%
\end{center} 
\caption{The $k$-way normalised cut and cumulative times taken by the eigenvector computations on \texttt{Citation} using 5 day intervals.}
\label{fig:citation5day}
\end{figure}

\section{Discussion}\label{sec:conclusion}

We have presented a novel incremental method for spectral graph clustering which updates the eigenvectors of the Laplacian in a computationally efficient way. Such an algorithm is useful for finding clusterings in time evolving graphs such as biological and social networks, and the Internet. A key part of the algorithm is a general way to approximate the first $k$ eigenvectors of a perturbed symmetric matrix given the eigenvectors of the original matrix. The resulting clustering algorithm, IASC, can be easily implemented using a standard linear algebra library. 

We analysed IASC in both theoretical and empirical respects. Using perturbation theory, we showed when the canonical angles between the real and approximate subspaces generated by our update algorithm are close. Furthermore, IASC is examined empirically relative to the computation of exact eigenvectors for each graph, the method of Ning et al., randomised SVD, and the Nystr\"{o}m approach. On 2 toy and 3 real datasets we show that IASC can often match the cluster accuracy of the exact approach using a small fraction of the total number of eigenvectors and at a much reduced computational cost. 

This work has opened up several perspectives for further study. The first is the analysis of the update of eigenvectors for a modularity matrix and other cluster quality criteria. As we have shown, the quality of the updates would depend on the spectrum of the matrices in question. Another interesting line of research is the issue of how to choose the number of clusters in the time evolving graphs. 


\bibliographystyle{plain}
\bibliography{references}

\begin{thebibliography}{10}

\bibitem{Auvert07}
B.~Auvert, H.~de~Arazoza, S.~Cl\'emen\c{c}on, J.~Perez, and R.~Lounes.
\newblock The {H}{I}{V}/{A}{I}{D}{S} epidemic in {C}uba: description and
  tentative explanation of its low {H}{I}{V} prevalence.
\newblock {\em BMC Infectious Diseases}, 7(30), November 2007.

\bibitem{berry1995using}
M.W. Berry, S.T. Dumais, and G.W. O'Brien.
\newblock Using linear algebra for intelligent information retrieval.
\newblock {\em SIAM review}, 37(4):573--595, 1995.

\bibitem{chung1997spectral}
F.R.K. Chung.
\newblock {\em Spectral graph theory}.
\newblock Amer Mathematical Society, 1997.

\bibitem{davis1970rotation}
C.~Davis and W.M. Kahan.
\newblock The rotation of eigenvectors by a perturbation. iii.
\newblock {\em SIAM Journal on Numerical Analysis}, 7(1):1--46, 1970.

\bibitem{dhanjal11cluster}
Charanpal Dhanjal, Romaric Gaudel, and St\'{e}phan Cl\'{e}men\c{c}on.
\newblock Incremental spectral clustering with the normalised laplacian.
\newblock In {\em 3rd NIPS Workshop on Discrete Optimization in Machine
  Learning}, 2011.

\bibitem{drineas2006fast}
P.~Drineas, R.~Kannan, and M.W. Mahoney.
\newblock Fast monte carlo algorithms for matrices ii: Computing a low-rank
  approximation to a matrix.
\newblock {\em SIAM Journal on Computing}, 36(1):158--183, 2006.

\bibitem{drineas2005nystrom}
P.~Drineas and M.W. Mahoney.
\newblock On the nystr{\"o}m method for approximating a gram matrix for
  improved kernel-based learning.
\newblock {\em The Journal of Machine Learning Research}, 6:2153--2175, 2005.

\bibitem{erdos1959random}
Paul Erd\H{o}s and Alfr{\'e}d R{\'e}nyi.
\newblock On random graphs.
\newblock {\em Publicationes Mathematicae}, 6:290--297, 1959.

\bibitem{flake2004graph}
G.W. Flake, R.E. Tarjan, and K.~Tsioutsiouliklis.
\newblock Graph clustering and minimum cut trees.
\newblock {\em Internet Mathematics}, 1(4):385--408, 2004.

\bibitem{fortunato2010community}
S.~Fortunato.
\newblock Community detection in graphs.
\newblock {\em Physics Reports}, 486(3-5):75--174, 2010.

\bibitem{fowlkes2004spectral}
C.~Fowlkes, S.~Belongie, F.~Chung, and J.~Malik.
\newblock Spectral grouping using the nystrom method.
\newblock {\em Pattern Analysis and Machine Intelligence, IEEE Transactions
  on}, 26(2):214--225, 2004.

\bibitem{gehrke2003overview}
J.~Gehrke, P.~Ginsparg, and J.~Kleinberg.
\newblock Overview of the 2003 kdd cup.
\newblock {\em ACM SIGKDD Explorations Newsletter}, 5(2):149--151, 2003.

\bibitem{haemers2004enumeration}
W.H. Haemers and E.~Spence.
\newblock Enumeration of cospectral graphs.
\newblock {\em European Journal of Combinatorics}, 25(2):199--211, 2004.

\bibitem{halko2011finding}
Nathan Halko, Per-Gunnar Martinsson, and Joel~A Tropp.
\newblock Finding structure with randomness: Probabilistic algorithms for
  constructing approximate matrix decompositions.
\newblock {\em SIAM review}, 53(2):217--288, 2011.

\bibitem{hotelling33pca}
Harold Hotelling.
\newblock Analysis of a complex of statistical variables into principle
  components.
\newblock {\em Journal of Educational Psychology}, 24:417--441 and 498--520,
  1933.

\bibitem{kong2011fast}
T.~Kong, Y.~Tian, and H.~Shen.
\newblock A fast incremental spectral clustering for large data sets.
\newblock In {\em Parallel and Distributed Computing, Applications and
  Technologies (PDCAT), 2011 12th International Conference on}, pages 1--5.
  IEEE, 2011.

\bibitem{kumar2009sampling}
S.~Kumar, M.~Mohri, and A.~Talwalkar.
\newblock Sampling techniques for the nystr{\"o}m method.
\newblock In {\em Conference on Artificial Intelligence and Statistics}, pages
  304--311, 2009.

\bibitem{kwok03incrementalEig}
James~T. Kwok and Zhao Haitao.
\newblock {Incremental eigen-decomposition }.
\newblock In {\em Proceedings of the International Conference on Artificial
  Neural Networks (ICANN)}, pages 270--273, 2003.

\bibitem{lehoucq1998arpack}
R.B. Lehoucq, D.C. Sorensen, and C.~Yang.
\newblock {\em ARPACK users' guide: solution of large-scale eigenvalue problems
  with implicitly restarted Arnoldi methods}, volume~6.
\newblock Siam, 1998.

\bibitem{li2011time}
M.~Li, X.C. Lian, J.T. Kwok, and B.L. Lu.
\newblock Time and space efficient spectral clustering via column sampling.
\newblock In {\em Computer Vision and Pattern Recognition (CVPR), 2011 IEEE
  Conference on}, pages 2297--2304. IEEE, 2011.

\bibitem{mahdavi2012improved}
M.~Mahdavi, T.~Yang, and R.~Jin.
\newblock An improved bound for the nystrom method for large eigengap.
\newblock {\em arXiv preprint arXiv:1209.0001}, 2012.

\bibitem{newman04fastAlg}
M.~E.~J. Newman.
\newblock Fast algorithm for detecting community structure in networks.
\newblock {\em Physical Review E}, 69(6):066133, Jun 2004.

\bibitem{newman2006modularity}
M.E.J. Newman.
\newblock Modularity and community structure in networks.
\newblock {\em Proceedings of the National Academy of Sciences},
  103(23):8577--8582, 2006.

\bibitem{newman2004finding}
M.E.J. Newman and M.~Girvan.
\newblock Finding and evaluating community structure in networks.
\newblock {\em Physical review E}, 69(2):026113, 2004.

\bibitem{ng2001spectral}
A.~Ng, M.~Jordan, and Y.~Weiss.
\newblock {On spectral clustering: Analysis and an algorithm}.
\newblock In {\em Advances in Neural Information Processing Systems 14:
  Proceeding of the 2001 Conference}, pages 849--856, 2001.

\bibitem{ning2007incremental}
H.~Ning, W.~Xu, Y.~Chi, Y.~Gong, and T.~Huang.
\newblock Incremental spectral clustering with application to monitoring of
  evolving blog communities.
\newblock In {\em SIAM Int. Conf. on Data Mining}. Citeseer, 2007.

\bibitem{ning2010incremental}
H.~Ning, W.~Xu, Y.~Chi, Y.~Gong, and T.S. Huang.
\newblock Incremental spectral clustering by efficiently updating the
  eigen-system.
\newblock {\em Pattern Recognition}, 43(1):113--127, 2010.

\bibitem{o1994information}
G.W. O'Brien.
\newblock Information management tools for updating an svd-encoded indexing
  scheme.
\newblock {\em Master's thesis, The University of Knoxville, Tennessee,
  Knoxville, TN}, 1994.

\bibitem{Rand71}
W.~Rand.
\newblock {Objective criteria for the evaluation of clustering methods}.
\newblock {\em Journal of the American Statistical Association},
  66(336):846--850, 1971.

\bibitem{richard10}
Emile Richard, Nicolas Baskiotis, Theodoros Evgeniou, and Nicolas Vayatis.
\newblock Link discovery using graph feature tracking.
\newblock In J.~Lafferty, C.~K.~I. Williams, J.~Shawe-Taylor, R.S. Zemel, and
  A.~Culotta, editors, {\em Proceedings of the 23$^{rd}$ Annual Conference on
  Neural Information Processing Systems (NIPS'10)}, pages 1966--1974, 2010.

\bibitem{satuluri2009scalable}
V.~Satuluri and S.~Parthasarathy.
\newblock Scalable graph clustering using stochastic flows: applications to
  community discovery.
\newblock In {\em Proceedings of the 15th ACM SIGKDD international conference
  on Knowledge discovery and data mining}, pages 737--746. ACM, 2009.

\bibitem{schaeffer2007graph}
S.E. Schaeffer.
\newblock Graph clustering.
\newblock {\em Computer Science Review}, 1(1):27--64, 2007.

\bibitem{sch98nonlinear}
Bernhard Sch{\"o}lkopf, Alex~J. Smola, and Klaus-Robert M{\"u}ller.
\newblock Nonlinear component analysis as a kernel eigenvalue problem.
\newblock {\em Neural Computation}, 10(5):1299--1319, 1998.

\bibitem{shi2000normalized}
J.~Shi and J.~Malik.
\newblock {Normalized cuts and image segmentation}.
\newblock {\em Pattern Analysis and Machine Intelligence, IEEE Transactions
  on}, 22(8):888--905, 2000.

\bibitem{stewart1990matrix}
G.W. Stewart and J.~Sun.
\newblock {\em Matrix perturbation theory}, volume 175.
\newblock Academic press New York, 1990.

\bibitem{valgren2007incremental}
C.~Valgren, T.~Duckett, and A.~Lilienthal.
\newblock Incremental spectral clustering and its application to topological
  mapping.
\newblock In {\em Robotics and Automation, 2007 IEEE International Conference
  on}, pages 4283--4288. IEEE, 2007.

\bibitem{von2007tutorial}
U.~Von~Luxburg.
\newblock {A tutorial on spectral clustering}.
\newblock {\em Statistics and Computing}, 17(4):395--416, 2007.

\bibitem{weyl1912asymptotische}
H.~Weyl.
\newblock Das asymptotische verteilungsgesetz der eigenwerte linearer
  partieller differentialgleichungen (mit einer anwendung auf die theorie der
  hohlraumstrahlung).
\newblock {\em Mathematische Annalen}, 71(4):441--479, 1912.

\bibitem{williams01nystrom}
Christopher~K.~I. Williams and Matthias Seeger.
\newblock Using the {Nystr{\"o}m} method to speed up kernel machines.
\newblock In Todd~K. Leen, Thomas~G. Dietterich, and Volker Tresp, editors,
  {\em Advances in Neural Information Processing Systems 13}, pages 682--688,
  Cambridge, MA, 2000. {MIT} Press.

\bibitem{yan2009fast}
D.~Yan, L.~Huang, and M.I. Jordan.
\newblock Fast approximate spectral clustering.
\newblock In {\em Proceedings of the 15th ACM SIGKDD international conference
  on Knowledge discovery and data mining}, pages 907--916. ACM, 2009.

\bibitem{yun2011nystrom}
Jeong-Min Yun and Seungjin Choi.
\newblock Nystr{\"o}m approximations for scalable face recognition: A
  comparative study.
\newblock In {\em Neural Information Processing}, pages 325--334. Springer,
  2011.

\bibitem{zha1999updating}
H.~Zha and H.D. Simon.
\newblock {On updating problems in latent semantic indexing}.
\newblock {\em SIAM Journal on Scientific Computing}, 21:782, 1999.

\bibitem{zhao2006novel}
H.~Zhao, P.C. Yuen, and J.T. Kwok.
\newblock A novel incremental principal component analysis and its application
  for face recognition.
\newblock {\em Systems, Man, and Cybernetics, Part B: Cybernetics, IEEE
  Transactions on}, 36(4):873--886, 2006.

\end{thebibliography}

\end{document}